\documentclass{article}

\usepackage{graphicx}
\usepackage{latexsym,amsmath,amssymb,amsfonts,amstext,amsthm,mathrsfs,bm}
\usepackage[utf8]{inputenc} 
\usepackage[T1]{fontenc}    

\usepackage{hyperref}       

\usepackage{cleveref}

\usepackage{url}            
\usepackage{booktabs}       
\usepackage{amsfonts}       
\usepackage{nicefrac}       
\usepackage{microtype}      

\usepackage{subfigure}
\usepackage{epsfig}

\usepackage{caption}
\usepackage{dsfont}
\usepackage{enumitem}
\usepackage{thm-restate}
\usepackage{color}
\usepackage{mathtools}
\usepackage{bbm}

\usepackage{mathtools}

\usepackage{algorithm,algcompatible,amsmath}
\algnewcommand\INPUT{\item[\textbf{Input:}]}
\algnewcommand\OUTPUT{\item[\textbf{Output:}]}

\allowdisplaybreaks[4]

\newcommand{\mP}{\mathbb P}
\newcommand{\mB}{\mathcal B}

\newcommand{\ms}{\mathcal{S}}
\newcommand{\ma}{\mathcal{A}}

\newcommand{\mam}{\mathcal M}

\newcommand{\mE}{\mathbb E}

\newcommand{\mcs}{\mathcal S}
\newcommand{\mca}{\mathcal A}
\newcommand{\mf}{\mathcal F}
\newcommand{\mcv}{\mathcal V}

\newcommand{\ltwo}[1]{\left\|#1\right\|_2}
\newcommand{\lone}[1]{\left|#1\right|}

\newcommand{\lF}[1]{\left\|#1\right\|_F}

\newcommand{\mR}{\mathbb{R}}

\newtheorem{theorem}{Theorem}

\newtheorem{lemma}{Lemma}

\newtheorem{assumption}{Assumption}

\usepackage[round]{natbib}

\usepackage{apalike}

\usepackage{geometry}
\hypersetup{
	colorlinks,linkcolor=red,anchorcolor=blue,citecolor=blue}
\usepackage{multirow}

\usepackage{authblk}

\makeatletter
\newcommand{\printfnsymbol}[1]{%
	\textsuperscript{\@fnsymbol{#1}}%
}
\makeatother

\begin{document}
	
\title{Sample Complexity Bounds for Two Timescale Value-based Reinforcement Learning Algorithms}

\author[ ]{Tengyu Xu, Yingbin Liang}
\affil[ ]{Department of Electrical and Computer Engineering, The Ohio State University}
\affil[ ]{\{xu.3260, liang.889\}@osu.edu}

\date{}
\maketitle

\begin{abstract}
Two timescale stochastic approximation (SA) has been widely used in value-based reinforcement learning algorithms. In the policy evaluation setting, it can model the linear and nonlinear temporal difference learning with gradient correction (TDC) algorithms as linear SA and nonlinear SA, respectively. In the policy optimization setting, two timescale nonlinear SA can also model the greedy gradient-Q (Greedy-GQ) algorithm. In previous studies, the non-asymptotic analysis of linear TDC and Greedy-GQ has been studied in the Markovian setting, with diminishing or accuracy-dependent stepsize. For the nonlinear TDC algorithm, only the asymptotic convergence has been established. In this paper, we study the non-asymptotic convergence rate of two timescale linear and nonlinear TDC and Greedy-GQ under Markovian sampling and with accuracy-independent constant stepsize. For linear TDC, we provide a novel non-asymptotic analysis and show that it attains an $\epsilon$-accurate solution with the optimal sample complexity of $\mathcal{O}(\epsilon^{-1}\log(1/\epsilon))$ under a constant stepsize. For nonlinear TDC and Greedy-GQ, we show that both algorithms attain $\epsilon$-accurate stationary solution with sample complexity $\mathcal{O}(\epsilon^{-2})$. It is the first non-asymptotic convergence result established for nonlinear TDC under Markovian sampling and our result for Greedy-GQ outperforms the previous result orderwisely by a factor of $\mathcal{O}(\epsilon^{-1}\log(1/\epsilon))$.
\end{abstract}

\section{Introduction}

Two timescale stochastic approximation (SA) algorithms have wide applications in reinforcement learning (RL) \cite{sutton2018reinforcement}. Typically, two timescale SA algorithms involve iterations of two types of variables updated at different speeds, i.e., the stepsizes for two iterates are chosen differently so that one iterate runs much faster than the other \cite{borkar1997stochastic,borkar2009stochastic}. Such algorithms are widely used to solve both policy evaluation and policy optimization problems in RL, in which the goal of policy evaluation is to estimate the expected total reward (i.e. value function) of a target policy, and the goal of policy optimization is to search for a policy with the optimal expected total reward.

In the policy evaluation problem, temporal difference (TD) learning \cite{sutton1988learning} is one of the most widely used algorithms when a linear function class is utilized to approximate the value function. However, in the off-policy setting, in which the target policy to be evaluated is different from the behavior policy that generates samples, TD learning may diverge to infinity. To overcome such an issue, \cite{sutton2009fast} proposed the two timescale linear TD with gradient correction (TDC) algorithm, which has convergence guarantee in the off-policy setting. The two timescale linear TDC is a special case of two timescale linear SA, whose asymptotic convergence has been established in \cite{sutton2009fast,borkar2009stochastic} and \cite{yu2017convergence,tadic2004almost,yaji2016stochastic} for the i.i.d. and Markovian settings, respectively. The non-asymptotic convergence rate of two timescale linear TDC/SA has also been studied. In the i.i.d. setting, under diminishing stepsize, \cite{dalal2017finite} established the sample complexity of $\mathcal{O}(\epsilon^{-1.5})$, and an improved complexity of $\mathcal{O}(\epsilon^{-1})$ was later established in \cite{dalal2019tale}. In the Markovian setting, \cite{xu2019two} established the complexity of $\mathcal{O}(\epsilon^{-1.5}\log^2(1/\epsilon))$ under a diminishing stepsize, and \cite{gupta2019finite} established the complexity of $\mathcal{O}(\epsilon^{-1-\zeta}\log^2(1/\epsilon))$ under a $\epsilon$-dependent stepsize, where $\zeta$ can be an arbitrarily small positive constant. Recently, \cite{kaledin2020finite} provides a tighter complexity bound of $\mathcal{O}(\epsilon^{-1})$ for two timescale linear SA under a diminishing stepsize. Although having progressed significantly, existing convergence guarantee were established either under a {\em diminishing} stepsize or a drastically small {\em $\epsilon$-level} stepsize, which yield very slow convergence and are rarely used in practice.
\begin{list}{$\bullet$}{\topsep=0.ex \leftmargin=0.15in \rightmargin=0.in \itemsep =0.01in}
	\item {\em Thus, the first goal of this paper is to investigate the two timescale linear TDC {\bf under a constant stepsize (not $\epsilon$-dependent)}, which is commonly adopted in practice, and to provide the finite-sample convergence guarantee for such a case. This necessarily requires a new approach differently from the existing ones. }
\end{list}

When a nonlinear function is utilized to approximate the value function, TD learning still suffers from the divergence issue \cite{tsitsiklis1996analysis}. To address that, \cite{bhatnagar2009convergent} proposed the two timescale nonlinear TDC, which can be modeled as a two timescale nonlinear SA. 
The asymptotic convergence of two timescale nonlinear SA has been well established in \cite{borkar1997stochastic,tadic2004almost,karmakar2018two}. However, the non-asymptotic convergence of two timescale nonlinear SA has only been established in the i.i.d.\ setting under some restrict assumptions such as global (local) stability and local linearizion \cite{borkar2018concentration,mokkadem2006convergence}. 
So far, the non-asymptotic convergence performance of two timescale nonlinear TDC has not been studied under the general Markovian sampling.
\begin{list}{$\bullet$}{\topsep=0.ex \leftmargin=0.15in \rightmargin=0.in \itemsep =0.01in}
	\item {\em The second goal of this paper is to provide the first non-asymptotic convergence analysis for two timescale nonlinear TDC with a constant stepsize, under {\bf Markovian sampling}, and without restricted assumptions.}
\end{list}

Moreover, in the policy optimization problem, Q-learning \cite{watkins1992q} has been widely used and has achieved significant success in practice. However, in the function approximation setting, Q-learning does not have convergence guarantee \cite{baird1995residual} unless under some restricted regularity assumptions \cite{melo2008analysis, zou2019finite, cai2019neural}. In corresponding to this, \cite{maei2010gq} proposed the Greedy-GQ algorithm in the linear function approximation setting, in which the algorithm is guaranteed to converge to a locally optimal policy without restricted assumptions. Similarly to nonlinear TDC algorithms, Greedy-GQ also adopts a two timescale update scheme, and is a special case of two timescale nonlinear SA. Under single-sample update and Markovian sampling, \cite{wang2020finite} provided the non-asymptotic convergence rate of Greedy-GQ with diminishing stepsize, which achieves the complexity of $\mathcal{O}(\epsilon^{-3}\log(\epsilon^{-1}))$. However, such a rate does not attain the typical complexity order of nonconvex optimization, and can be potentially improved with a larger stepsize.
\begin{list}{$\bullet$}{\topsep=0.ex \leftmargin=0.15in \rightmargin=0.in \itemsep =0.01in}
	\item {\em The last focus of this paper is to provide an improved non-asymptotic convergence rate for two timescale Greedy-GQ under a constant stepsize.}
\end{list}

\subsection{Our Contributions}

For two timescale linear TDC, we show that it achieves the sample complexity of $\mathcal{O}(\epsilon^{-1}\log(\epsilon^{-1}))$, which has the optimal dependence on $\epsilon$ due to the lower bound given in \cite{dalal2019tale}. Such a rate has been established in \cite{kaledin2020finite}, but only under a {\em diminishing} stepsize, which is rarely used in practice due to the slow empirical performance. In contrast, our guarantee is established under a {\bf constant (not $\epsilon$-dependent) stepsize}, which is commonly used in practice. Our analysis approach leverages the mini-batch sampling for each iteration to control the convergence error, which is significantly different from that in \cite{kaledin2020finite}, and can be of independent interest.

For two timescale nonlinear TDC, we establish the first non-asymptotic convergence rate under {\bf Markovian sampling}. We show that the mini-batch two timescale nonlinear TDC algorithm achieves the sample complexity of $\mathcal{O}(\epsilon^{-2})$.

For two timescale Greedy-GQ, we 
show that mini-batch two timescale Greedy-GQ with a constant stepsize and under Markovian sampling achieves the sample complexity of $\mathcal{O}(\epsilon^{-2})$. Our result orderwisely outperforms the previous result of Greedy-GQ with diminishing stepsize in \cite{wang2020finite} by a factor of $\mathcal{O}(\epsilon^{-1}\log(1/\epsilon))$.

\subsection{Related Work}
Due to the vast amount of studies on SA and value-based RL algorithms, we include here only the studies that are highly related to our work.

\textbf{Policy evaluation with {\em linear} function approximation.} In the on-policy setting, TD learning \cite{sutton1988learning} has been proposed to solve the policy evaluation problem in the linear function approximation setting. The non-asymptotic convergence rate of TD learning has been established in \cite{dalal2018finite} for the i.i.d. setting and in \cite{bhandari2018finite,srikant2019finite,hu2019characterizing} for the Markovian setting. \cite{cai2019neural} explored the linearizable structure of neural networks in the overparameterized regime, and studied the non-asymptotic convergence rate of TD learning with neural network approximation.
\cite{zou2019finite} studied the convergence rate of SARSA with linear function approximation in the Markovian setting, which can been viewed as a policy evaluation with dynamic changing transition kernel.

In the off-policy setting, GTD, GTD2 and TDC have been proposed to solve the divergence issue of TD learning \cite{sutton2008convergent,sutton2009fast,maei2011gradient}. The convergence rate of one timescale GTD and GTD2 algorithms has been established in \cite{liu2015finite} by converting the objective into a convex-concave saddle problem in the i.i.d. setting, and was further generalized to the Markovian setting in \cite{wang2017finite}. For two timescale linear TDC, in the i.i.d.\ setting, the non-asymptotic analysis was provided in \cite{dalal2017finite,dalal2019tale}. In the Markovian setting, the non-asymptotic convergence rate was first established in \cite{xu2019two} under diminishing stepsize and in \cite{gupta2019finite} under constant stepsize. The result in \cite{xu2019two} was later improved by \cite{kaledin2020finite} to achieve the optimal convergence rate.

\textbf{Policy evaluation with {\em nonlinear} function approximation.}
Two timescale nonlinear TDC is proposed by \cite{bhatnagar2009convergent}, in which a smooth nonlinear function is utilized to approximate the value function. Nonlinear TDC with i.i.d.\ samples is a special case of two time-scale nonlinear SA with martingale noise, whose asymptotic convergence has been established in \cite{bhatnagar2009convergent,maei2011gradient} by using asymptotic convergence results in nonlinear SA \cite{borkar1997stochastic,borkar2009stochastic,tadic2004almost}. Under the global/local asymptotic stability assumptions or local linearizion assumption, the non-asymptotic convergence of two timescale nonlinear SA with martingale noise has been studied in \cite{borkar2018concentration}. Under certain stability assumptions, the asymptotic convergence of two timescale nonlinear SA with Markov noise was established in \cite{karmakar2016asymptotic,karmakar2018two}. 
A concurrent study \cite{qiu2020single} also investigated nonlinear TDC and obtained the same sample complexity of $\mathcal{O}(\epsilon^{-2})$ as our result. However, \cite{qiu2019finite} only considered the i.i.d. setting, whereas we considered the more general Markovian setting.

\textbf{Policy optimization with {\em linear} function approximation.}
Q-learning \cite{watkins1992q} is one of the most widely used value-based policy optimization algorithms. The asymptotic and non-asymptotic convergence have been established for Q-learning with linear function approximation in \cite{melo2008analysis} and \cite{zou2019finite}, respectively, under certain regularity assumption. Under a similar regularity assumption, \cite{cai2019neural} established the convergence rate of Q-Learning in the neural network approximation setting. However, without regularity assumptions, Q-Learning does not have convergence guarantee in the function approximation setting.
\cite{maei2010toward} proposed two timescale Greedy-GQ to solve the divergence issue of Q-Learning with linear function approximation, and the asymptotic convergence of Greedy-GQ was also established therein. Recently, \cite{wang2020finite} studied the non-asymptotic convergence rate of Greedy-GQ under diminishing stepsize in the Markovian setting. In this paper, we provide an orderwisely better convergence rate than that in \cite{wang2020finite}.

\section{Markov Decision Process}\label{sc: mdp}

Consider a Markov decision process (MDP) denoted $(\mcs, \mca, \mathsf{P},r,\gamma)$. Here, $\mcs\subset \mR^d$ is a state space, $\mca$ is an action set, $\mathsf{P}=\mathsf{P}(s^\prime|s,a)$ is the transition kernel,  $r(s, a, s^\prime)$ is the reward function bounded by $r_{\max}$, and $\gamma\in(0,1)$ is the discount factor. A stationary policy $\pi$ maps a state $s\in \mcs$ to a probability distribution $\pi(\cdot|s)$ over the action space $\mca$. At time-step $t$, suppose the process is in some state $s_t\in \mcs$. Then an action $a_t\in \mca$ is taken based on the distribution $\pi(\cdot|s_t)$, the system transitions to a next state $s_{t+1}\in \mcs$ governed by the transition kernel $\mathsf{P}(\cdot|s_t,a_t)$, and a reward $r_t=r(s_t, a_t, s_{t+1})$ is received. We assume the associated Markov chain $p(s^\prime|s)=\sum_{a\in\mca}p(s^\prime|s,a)\pi(a|s)$ is ergodic, and let $\mu_\pi$ be the induced stationary distribution of this MDP, i.e., $\sum_{s}p(s^\prime|s)\mu_{\pi}(s)=\mu_{\pi}(s^\prime)$. The state value function for policy $\pi$ is defined as: $V^\pi\left(s\right)=\mE[\sum_{t=0}^{\infty}\gamma^t r(s_t,a_t, s_{t+1})|s_0=s,\pi]$, and the state-action value function is defined as: $Q^\pi(s,a)=\mE[\sum_{t=0}^{\infty}\gamma^t r(s_t,a_t, s_{t+1})|s_0=s, a_0=a, \pi]$. It is known that $V^\pi(s)$ is the unique fixed point of the Bellman operator $T^\pi$, i.e., $V^\pi(s) = T^\pi V^\pi(s)\coloneqq r^\pi(s)+\gamma\mE_{s^\prime |s} V^\pi(s^\prime)$, where $r^\pi(s)=\mE_{a, s^\prime|s}r(s,a,s^\prime)$ is the expected reward of the Markov chain induced by the policy $\pi$. We take the following standard assumption for the MDP in this paper, which has also been adopted in previous works \cite{bhandari2018finite,zou2019finite,srikant2019finite,xu2019two,xu2020reanalysis}.

\begin{assumption}[Geometric ergodicity]\label{ass1}
	There exist constants $\kappa>0$ and $\rho\in(0,1)$ such that
	\begin{flalign*}
	\sup_{s\in\mcs}d_{TV}(\mP(s_t|s_0=s),\mu_{\pi_b})\leq \kappa\rho^t, \forall t\geq 0,
	\end{flalign*}
	where $\mP(s_t|s_0=s)$ is the distribution of $s_t$ conditioned on $s_0=s$ and $d_{TV}(P,Q)$ denotes the total-variation distance between the probability measures $P$ and $Q$.
\end{assumption}
\Cref{ass1} holds for any time-homogeneous Markov chain with finite state space and any uniformly ergodic Markov chain with general state space \cite{bhandari2018finite,zou2019finite,xu2019two}.

\section{Two Timescale TDC with Linear Function Approximation}\label{sc: linearTDC}
In this section we first introduce the two timescale linear TDC algorithm to solve the policy evaluation problem, and then present our convergence rate result.

\subsection{Algorithm}

When $\mcs$ is large or infinite, a linear function $\hat{v}(s,\theta)=\phi(s)^\top\theta$ is often used to approximate the value function $V^\pi(s)$, where $\phi(s)\in\mR^d$ is a fixed feature vector for state $s$ and $\theta\in\mR^d$ is a parameter vector. We can also write the linear approximation in the vector form as $\hat{v}(\theta)={\rm \Phi} \theta$, where ${\rm \Phi}$ is the $|\mcs|\times d$ feature matrix. Without loss of generality, we assume that the feature vector $\ltwo{\phi(s)}\leq 1$ for all $s\in\mcs$ and the columns of the feature matrix $\Phi$ are linearly independent. 
Here we consider policy evaluation problem in the off-policy setting. Namely, a sample path $\{ (s_t,a_t,s_{t+1}) \}_{t\geq 0}$ is generated by the Markov chain according to a behavior policy $\pi_b$, but our goal is to obtain the value function of a target policy $\pi$, which is different from $\pi_b$. 

To find a parameter $\theta^*\in\mR^d$ with $ \mE_{\mu_{\pi_b}}\hat{v}(s,\theta^*)= \mE_{\mu_{\pi_b}}T^\pi\hat{v}(s,\theta^*)$. The linear TDC algorithm \cite{sutton2009fast} updates the parameter by minimizing the mean-square projected Bellman error (MSPBE) objective, defined as
\begin{flalign*}
J(\theta)&=\mE_{\mu_{\pi_b}}[\hat{v}(s,\theta)-{\rm \Pi} T^\pi\hat{v}(s,\theta)]^2,
\end{flalign*}
where ${\rm \Pi}$ is the orthogonal projection operation onto the function space $\hat{\mcv}=\{\hat{v}(\theta)\  |\  \theta\in\mR^d\  \text{and}\ \hat{v}(\cdot, \theta)=\phi(\cdot)^\top \theta   \}$. When the columns of the feature matrix $\Phi$ are linearly independent, \cite{sutton2009fast} shows that $J(\theta)$ is strongly convex and has $\theta^*=-A^{-1}b$ as its global minimum, i.e., $J(\theta^*)=0$, where $A=\mE_{\mu_{\pi_b}}[(\gamma\mE_\pi[\phi(s^\prime)|s]-\phi(s))\phi(s)]$ and $b=\mE_{\mu_{\pi_b}}[\mE_\pi[r(s,a,s^\prime)|s]\phi(s)]$. A convenient way to find $\theta^*$ is to minimize the MSPBE objective function $J(\theta)$ using the gradient descent method: $\theta_{t+1}=\theta_t - \frac{\alpha}{2} \nabla J(\theta_t)$, where $\alpha>0$ is the stepsize and the gradient $\nabla J(\theta)$ was derived by \cite{bhatnagar2009convergent} as follows:
{\begin{flalign}\label{linear-TDCupdate}
	-\frac{1}{2}\nabla J(\theta) = \mE_{\mu_{\pi_b}}[\mE_\pi[\delta(\theta)|s]\phi(s)]-\gamma\mE_{\mu_{\pi_b}}[ \mE_\pi[\phi(s^\prime)|s]\phi(s)^\top ]w(\theta),
	\end{flalign}}
where $\delta(\theta)=r(s,a,s^\prime) + \gamma \hat{v}(s^\prime,\theta) - \hat{v}(s,\theta)$ is the temporal difference error, $w(\theta)\coloneqq \mE_{\mu_{\pi_b}}[\phi(s)\phi(s)^\top]^{-1}\mE_{\mu_{\pi_b}}[ \mE_\pi[\delta(\theta)|s]\phi(s)]$.
In practice, stochastic gradient descent (SGD) method is usually adopted to perform the update in \cref{linear-TDCupdate} approximately. However, directly sampling is not applicable to $w(\theta)$. To solve such an issue, an auxiliary parameter $w_t$ can be introduced to estimate the vector $w(\theta_t)$, i.e., $w_t\approx w(\theta_t)$, by solving a linear SA with the following corresponding ODE:
\begin{flalign*}
\dot{w}=-\mE_{\mu_{\pi_b}}[\phi(s)\phi(s)^\top]w + \mE_{\mu_{\pi_b}}[\mE_\pi[\delta(\theta)|s]\phi(s)].
\end{flalign*}

Given $w_t$, the parameter $\theta_t$ can then be updated with a stochastic approximation of $\nabla J(\theta_t)$ obtained via directly sampling:
\begin{flalign}
\theta_{t+1}=\theta_t + \alpha \frac{1}{\lone{\mB_t}}\sum_{j\in \mB_t}g(\theta_t, w_t,x_j),\label{eq: linearTDC_stoc}
\end{flalign}
where $\mB_t$ is the mini-batch sampled from the MDP, $g(\theta_t, w_t,x_j)=\rho(s_j, a_j)(\delta_j(\theta_t)\phi(s_j)-\gamma\phi(s_{j+1})\phi(s_j)^\top w_t)$, $\rho(s,a)=\pi(a|s)/\pi_b(a|s)$ is the importance weighting factor with $\rho_{\max}$ being its maximum value, and $x_j$ denotes the sample $(s_j,a_j,s_{j+1})$.

\Cref{al: linearTDC} is an {\bf online }algorithm based on a {\bf single sample path}. \Cref{al: linearTDC} adopts a two timescale update scheme, in which parameters $\theta_t$ and $w_t$ are updated simultaneously but with different stepsizes. Specifically, the main parameter $\theta_t$ iterates at a slow timescale with a smaller stepsize, and the auxiliary parameter $w_t$ iterates at a fast timescale with a larger stepsize. By doing so, $w_t$ can be close to $w(\theta_t)$ asymptotically, so that $\theta_t$ is updated approximately in the direction of $-\nabla J(\theta)$. 
\Cref{al: linearTDC} utilizes an accuracy-independent constant stepsize, i.e., $\alpha, \beta=\mathcal{O}(1)$ for both the updates of $\theta_t$ and $w_t$, and a mini-batch of samples $\{ (s_j, a_j, s_{j+1}) \}_{i_t\leq j\leq i_t+M-1}$ are taken sequentially from the trajectory at each iteration to perform the update. As we will show later, linear TDC in this setting is guaranteed to converge to the global optimal with an arbitrary accuracy level.

\begin{algorithm}[H]
	\null
	\caption{Two Timescale Linear TDC}
	\label{al: linearTDC}
	\begin{algorithmic}[1]
		\STATE {\bfseries Input:} batch size $M$, learning rate $\alpha$ and $\beta$
		\STATE {\bfseries Sampling:} A trajectory $\{ s_j, a_j\}_{j\geq 0}$ is sampled by following the behaviour policy $\pi_b$
		\STATE \textbf{Initialization: } $\theta_0$ and $w_0$
		\FOR{$t=0,\cdots,T-1$}
		\STATE $i_t=tM$
		\STATE $w_{t+1} = w_t + \beta\frac{1}{M}\sum_{j=i_t}^{i_t+M-1}(-\phi(s_j)\phi(s_j)^\top w_t + \rho(s_j, a_j)\delta_j(\theta_t)\phi(s_j))$
		\STATE $\theta_{t+1}=\theta_t + \alpha\frac{1}{M}\sum_{j=i_t}^{i_t+M-1} \rho(s_j, a_j)(\delta_j(\theta_t)\phi(s_j)-\gamma \phi(s_{j+1})\phi(s_j)^\top w_t)$
		\ENDFOR
		\STATE {\bfseries Output:} $\theta_T$
	\end{algorithmic}
\end{algorithm}

\subsection{Convergence Analysis}
We define matrix $C=-\mE_{\mu_{\pi_b}}[\phi(s)\phi(s)^\top]$. Let $\lambda_1=\lone{\lambda_{\max}(A^\top C^{-1} A)}$, $\lambda_2=\lone{\lambda_{\max}(C)}$ and $R_\theta=\ltwo{\theta^*}$. The following theorem provides the convergence rate and sample complexity of \Cref{al: linearTDC}.
\begin{theorem}\label{thm1}
	Suppose \Cref{ass1} hold. Consider \Cref{al: linearTDC} of two timescale linear TDC update. Let the stepsize $\alpha\leq \min\left\{ \frac{1}{8\lambda_1}, \frac{\lambda_1\lambda_2}{12}, \frac{\sqrt{\lambda_2\beta}}{4\sqrt{6}\rho_{\max}} \frac{\lambda_2\sqrt{\lambda_2}\beta}{16\rho^2_{\max}}, \frac{\lambda_1\lambda_2\beta}{64\rho^2_{\max}}, \frac{\lambda_1\lambda^2_2\beta}{768}  \right\}$, $\beta \leq  \min\left\{ \frac{1}{8\lambda_2}, \frac{\lambda_2}{4}\right\}$ and the batch size $M\geq 128\left( \rho^2_{\max} + \frac{1}{\lambda^2_2} \right)\frac{1+(\kappa-1)\rho}{1-\rho}\max\left\{1,\frac{8\beta+8\lambda_2\beta^2}{\lambda_1\lambda_2\alpha}, \frac{8+12\lambda_1\alpha}{\lambda_1} \right\}$. Then we have
	{\small \begin{flalign}
		&\mE[\ltwo{\theta_{T}-\theta^*}^2] \leq \left(1- \frac{\min\{  \lambda_1\alpha, \lambda_2\beta \}}{8} \right)^{T}\Delta_0 + \frac{A_1}{M},\label{eq: thm1_ub}
		\end{flalign}}where $\Delta_0 = \ltwo{w_0 - w^*(\theta_0)}^2 + \ltwo{\theta_0 - \theta^*}^2$, where $A_1$ is a constant defined in \cref{eq: 32} in \Cref{sc: proflemma1}. Furthermore, let $M\geq \frac{2A_1}{\epsilon}$ and $T\geq \frac{8}{\min\{  \lambda_1\alpha, \lambda_2\beta \}}\ln\left( \frac{2\Delta_0}{\epsilon} \right)$. The total sample complexity for \Cref{al: linearTDC} to achieve an $\epsilon$-accurate optimal solution $\theta^*$, i.e., $\mE[\ltwo{\theta_T-\theta^*}^2]\leq \epsilon$, is given by
	{\small \begin{flalign*}
		TM=\Theta\left( \frac{1}{\epsilon}\log\left( \frac{1}{\epsilon} \right) \right).
		\end{flalign*}}
\end{theorem}
\Cref{thm1} shows that the convergence error of \Cref{al: linearTDC} consists of two terms: the first term is the transient error decreasing at an exponential rate, and the second term is the variance error that diminishes as the batch size $M$ increases. This is in contrast to the single-sample TDC under constant stepsizes, which suffers from the variance and bias errors with order $\mathcal{O}(\beta^2/\alpha)$ \cite{gupta2019finite}. Thus, $\epsilon$-level small stepsizes $\alpha$ and $\beta$ are required in single-sample TDC to reduce the variance error to achieve the required $\epsilon$-accurate optimal solution, which can slow down the practical convergence speed significantly. In contrast, mini-batch TDC can attain high accuracy with a large constant (not $\epsilon$-level) stepsize. Our result of $\mathcal{O}(\epsilon^{-1}\log(1/\epsilon))$ achieves the optimal complexity order due to the lower bound given in \cite{dalal2019tale}. In contrast to the same sample complexity established in \cite{kaledin2020finite}, which is applicable only under a diminishing stepsize, our result given in \Cref{thm1} is applicable under the constant stepsize, which is practically preferred due to the much better performance.

We next provide a sketch of the proof for \Cref{thm1}. 
\begin{proof}[{\bf Proof Sketch of \Cref{thm1}}]
	The proof of \Cref{thm1} consists of the following three steps. At $t$-th step, we call $\ltwo{\theta_t-\theta^*}^2$ as the training error and $\ltwo{w_t - w(\theta_t)}^2$ as the tracking error.
	
	\noindent\textbf{Step 1:} We establish the following induction relationships for the tracking error:
	\begin{flalign}\label{eq: sketch1}
	\mE\left[ \ltwo{w_{t+1}-w(\theta_{t+1})}^2 \right]&\leq (1 - \Theta(\lambda_2\beta) +  \Theta(\alpha^2/\beta) )\mE\left[ \ltwo{w_{t}-w(\theta_{t})}^2 \right] \nonumber\\
	&\quad + \Theta(\alpha^2/\beta + \lambda_1\alpha)\mE[\ltwo{\theta_t - \theta^*}^2] + \Theta(1/M).
	\end{flalign}
	
	\noindent\textbf{Step 2:} We then establish the induction relationships for the training error:
	\begin{flalign}\label{eq: sketch2}
	\mE\left[\ltwo{\theta_{t+1}-\theta^*}^2\right]&\leq ( 1 - \Theta(\lambda_1\alpha) + \Theta( \alpha^2))\mE\left[\ltwo{\theta_{t}-\theta^*}^2\right] \nonumber\\
	&\quad + \Theta(\alpha + \alpha^2)\mE\left[ \ltwo{w_{t}-w(\theta_{t})}^2 \right] + \Theta(1/M).
	\end{flalign}
	
	\noindent\textbf{Step 3:} Combing \cref{eq: sketch1} and \cref{eq: sketch2} and letting the stepsize $\alpha$ and $\beta$ and batch size $M$ satisfy the requirement specified in \Cref{thm1}, we establish the induction relationship of $\Delta_t =\mE[\ltwo{\theta_{t}-\theta^*}^2] + \mE[ \ltwo{w_{t}-w(\theta_{t})}^2]$ as follows:
	\begin{flalign}\label{eq: sketch3}
	\Delta_{t+1} \leq \left(1- \Theta(\min\{  \lambda_1\alpha, \lambda_2\beta \}) \right) \Delta_t + \Theta( 1/M).
	\end{flalign}
	Applying \cref{eq: sketch3} recursively from $t=T-1$ to $0$ yields the desired convergence result.
\end{proof}

\section{Two Timescale TDC with Nonlinear Function Approximation}\label{sc: nonlinearTDC}

In this section we first introduce the nonlinear two timescale TDC algorithm to solve the policy evaluation problem, then we provide our non-asymptotic convergence rate result.
\subsection{Algorithm}\label{sc: nonlineartd}

In this section we consider policy evaluation problem with nonlinear function approximation, in which a parameterized smooth {\em nonlinear} function $\hat{v}(s,\theta)$ is used to approximate the value function $V^\pi(s)$. \cite{bhatnagar2009convergent} proposed an algorithm to find a parameter for the approximator $\hat{v}(s,\theta)$, named nonlinear TDC.
The nonlinear TDC updates the parameter by minimizing the following mean-square projected Bellman error objective defined as:
\begin{flalign}
J(\theta)&=\mE_{\mu_{\pi}}[\hat{v}(s,\theta)-{\rm \Pi}_\theta T^\pi \hat{v}(s,\theta)]^2,\label{mspbe1}
\end{flalign}
where ${\rm \Pi}_\theta$ is the orthogonal projection operation into the function space $\bar{\mcv}=\{\bar{V}(s,\zeta)\  |\  \zeta\in\mR^d\  \text{and}\ \bar{v}(s, \zeta)=\phi_\theta(s)^\top \zeta \,\text{with}\, (\phi_\theta(s))_i=\nabla_{\theta_i}\hat{v}(s,\theta)   \}$.
In general, since $J(\theta)$ defined in \cref{mspbe1} is nonconvex with respect to the parameter $\theta$, finding the global minimum of $J(\theta)$ is NP-hard. However, we can still apply gradient descent method to find a local optimum (i.e., first-order stationary point) of $J(\theta)$, via updating the parameter $\theta$ iteratively as $\theta_{t+1}=\theta_t - \frac{\alpha_t}{2} \nabla J(\theta_t)$, where $\alpha_t>0$ is the stepsize and the gradient $\nabla J(\theta)$ was derived by \cite{bhatnagar2009convergent} as follows:
{\begin{flalign}\label{nonlinear-TDCupdate}
	-\frac{1}{2}\nabla J(\theta) = \mE[\delta(\theta)\phi_\theta(s)]-\gamma\mE[\phi_\theta(s^\prime)\phi_\theta(s)^\top ]w(\theta)-h(\theta,w(\theta)),
	\end{flalign}}
where $\delta(\theta)=r(s,a,s^\prime) + \gamma \hat{v}(s^\prime,\theta)-\hat{v}(s,\theta)$ is the temporal difference and
{\begin{flalign*}
	&w(\theta)\coloneqq \mE[\phi_\theta(s)\phi_\theta(s)^\top]^{-1}\mE[\delta(\theta)\phi_\theta(s)],\nonumber\\ &h(\theta,u)\coloneqq \mE[(\delta(\theta)-\phi_\theta(s)^\top u)\nabla^2_\theta V_\theta(s)u].
	\end{flalign*}}
Similarly to linear TDC studied in \Cref{sc: linearTDC}, in order to estimate the gradient in \cref{nonlinear-TDCupdate}, an auxiliary parameter $w_t$ can be used to estimate the vector $w(\theta_t)$, i.e., $w_t\approx w(\theta_t)$, by solving a linear SA with the following corresponding ODE:
\begin{flalign}\label{eq: ODE1}
\dot{w}=-\mE[\phi_\theta(s)\phi_\theta(s)^\top]w + \mE[\delta(\theta)\phi_\theta(s)].
\end{flalign}
Given $w_t$, the parameter $\theta_t$ can then be updated with a stochastic approximation of $\nabla J(\theta_t)$ obtained via directly sampling:
\begin{flalign}
\theta_{t+1}&=\theta_t + \alpha_t\frac{1}{\lone{\mB_t}}\sum_{j\in \mB_t} g(\theta_t, w_t,x_j), \label{eq: nonlinearTDC_stoc}
\end{flalign}
where $\mB_t$ is the minibatch sampled from the MDP, $x_j$ denotes the sample $(s_j,a_j,s_{j+1})$ and we define $g(\theta_t, w_t,x_j)=\delta_j(\theta_t)\phi_{\theta_t}(s_j)-\gamma\phi_{\theta_t}(s_{j+1})\phi_{\theta_t}(s_j)^\top w_t-h_j(\theta_t,w_t)$, where $h_j(\theta_t,w_t)=(\delta_j(\theta_t)-\phi_{\theta_t}(s_j)^\top w_t)\nabla^2_\theta V_{\theta_t}(s_j)w_t$.
The nonlinear TDC algorithm is shown in \Cref{al: nonlinearTDC}. Similarly to \Cref{al: linearTDC}, here we also use a mini-batch of samples for each update.

\begin{algorithm}[H]
	\null
	\caption{Two Time-scale Nonlinear TDC}
	\label{al: nonlinearTDC}
	\begin{algorithmic}[1]
		\STATE \textbf{Input} batch size $M$, learning rate $\alpha$ and $\beta$
		\STATE {\bfseries Sampling:} A trajectory $\{ s_j, a_j\}_{j\geq 0}$ is sampled by following the policy $\pi$
		\STATE \textbf{Initialization: } $\theta_0$ and $w_0$
		\FOR{$t=0, 1, ..., T-1$}
		\STATE $i_t=tM$
		\STATE $w_{t+1} = w_t + \beta\frac{1}{M}\sum_{j=i_t}^{i_t+M-1}(-\phi_{\theta_t}(s_j)\phi_{\theta_t}(s_j)^\top w_t + \delta_j(\theta_t)\phi_{\theta_t}(s_j))$
		\STATE $\theta_{t+1}=\theta_t + \alpha\frac{1}{M}\sum_{j=i_t}^{i_t+M-1}(\delta_j(\theta_t)\phi_{\theta_t}(s_j)-\gamma\phi_{\theta_t}(s_{j+1})\phi_{\theta_t}(s_j)^\top w_t-h_j(\theta_t,w_t))$
		\ENDFOR
		\OUTPUT $\tilde{\theta}_{\hat{T}}$ with $\hat{T}$ chosen uniformly from $\{1,\cdots,T\}$
	\end{algorithmic}
\end{algorithm}

\subsection{Convergence Analysis}

Our analysis of \Cref{al: nonlinearTDC} will be based on the following assumptions.
\begin{assumption}[Bounded feature]\label{ass2}
	For any state $s\in \mathcal{S}$ and any vector $\theta\in \mR^d$, we have $\ltwo{\phi_\theta(s)}\leq C_\phi$, $\lone{V(s,\theta)}\leq C_v$ and $\lF{\nabla^2_\theta V(s,\theta)}\leq D_v$, where $C_\phi$, $C_v$ and $D_v$ are positive constants.
\end{assumption}
\begin{assumption}[Smoothness]\label{ass3}
	For any state $s\in \mathcal{S}$ and any vector $\theta, \theta^\prime\in \mR^d$, we have $\lone{V(s,\theta)-V(s,\theta^\prime)}\leq L_v\ltwo{\theta-\theta^\prime}$, $\ltwo{\phi_\theta(s)-\phi_{\theta^\prime}(s)}\leq L_\phi \ltwo{\theta-\theta^\prime}$, and $\ltwo{\nabla^2_\theta V(s,\theta) - \nabla^2_\theta V(s,\theta^\prime) } \leq L_h\ltwo{\theta-\theta^\prime} $, where $L_v$, $L_\phi$, and $L_h$ are positive constants.
\end{assumption}
\begin{assumption}[Non-singularity]\label{ass4}
	For any vector $\theta\in \mR^d$, we have $\text{eig}\{\mE[\phi_\theta(s) \phi_\theta(s)^\top]\} \geq \lambda_v$, where $\lambda_v$ is a positive constant.
\end{assumption}
\begin{assumption}[Lipschitz gradient]\label{ass5}
	For any vector $\theta, \theta^\prime$ and $w,w^\prime\in \mR^d$, and any sample $x$, we have $\ltwo{\nabla J(\theta)- \nabla J(\theta^\prime)}\leq L_J\ltwo{\theta-\theta^\prime}$, and $\ltwo{g(\theta, w,x)-g(\theta, w^\prime,x)}\leq L_e\ltwo{w-w^\prime}$ where $L_J$ and $L_e$ are positive constants.
\end{assumption}
Assumptions \ref{ass2}-\ref{ass5} are equivalent to the assumptions adopted in the original nonlinear TDC analysis \cite{bhatnagar2009convergent}, and can be satisfied by appropriately choosing the approximation function class $\hat{v}(s,\theta)$. The following theorem characterizes the converge rate and sample complexity of \Cref{al: nonlinearTDC}.

\begin{theorem}\label{thm2}
	Consider the two timescale nonlinear TDC algorithm in \Cref{al: nonlinearTDC}. Suppose Assumptions \ref{ass1}-\ref{ass5} hold. Let the stepsize $\beta\leq \min\{\frac{\lambda_v}{8C^4_\phi}, \frac{8}{\lambda_v} \}$ and $\alpha\leq \min\{\frac{1}{2L_J}, \frac{\lambda_v \beta}{8\sqrt{2}L_w L_e}, \frac{L_J \lambda^2_v\beta^2}{384 L^2_wL^2_e} \}$. We have
	\begin{flalign*}
	\mE\left[\ltwo{\nabla J(\theta_{\hat{T}})}^2\right]\leq \frac{8(J(\theta_0)-\mE[J(\theta_T)])}{\alpha T} + \frac{B_1\ltwo{w_0-w(\theta_0)}^2}{T} + \frac{B_2}{M},
	\end{flalign*}
	where $B_1$ and $B_2$ are constants defined in \Cref{sc: proflemma2} in \cref{eq: 33}. Furthermore, let $M\geq \frac{2	B_2}{\epsilon}$ and $T\geq \frac{2}{\epsilon}\left[ \frac{8J(\theta_0)}{\alpha} + B_1\ltwo{w_0-w(\theta_0)}^2 \right]$. The total sample complexity for \Cref{al: nonlinearTDC} to achieve an $\epsilon$-accurate stationary point, i.e., $\mE\big[ \ltwo{\nabla J(\theta_{\hat{T}})}^2 \big]\leq \epsilon$, is given by
	\begin{flalign*}
	TM=\Theta\left( \frac{1}{\epsilon^2} \right).
	\end{flalign*}
\end{theorem}

\Cref{thm2} shows that the convergence error of \Cref{al: nonlinearTDC} consists of three terms: the first two terms are the transient error decreasing at a sublinear rate as $T$ increases, and the third term contains the variance and bias errors that diminish as the batch size $M$ increases. 
We next provide a sketch of the proof for \Cref{thm2}.
\begin{proof}[{\bf Proof Sketch of \Cref{thm2}}]
	The proof of \Cref{thm2} consists of the following four steps.
	
	\noindent\textbf{Step 1:} We first provide \Cref{lemma: fixpointlip} to show that $w(\theta)$ is $L_w$-Lipschitz:
	\begin{flalign*}
	\ltwo{w(\theta)-w(\theta^\prime)}\leq L_w\ltwo{\theta - \theta^\prime},\quad \text{for all }\theta,\theta^\prime\in \mR^d.
	\end{flalign*}
	This property is crucial for the convergence analysis of two time-scale nonlinear TDC. It indicates that if $\theta_t$ changes slowly, then $w(\theta_t)$ also changes slowly. This allows our finite time analysis to be over a slowly changing linear SA with corresponding ODE defined in \cref{eq: ODE1}, guaranteeing that $\ltwo{w_t - w(\theta_t)}^2$ is small in an amortized sense.
	
	\noindent\textbf{Step 2:} We then establish the induction relationships for the tracking error $\ltwo{w_t - w(\theta_t)}^2$:
	\begin{flalign}\label{eq: sketch4}
	\mE\left[ \ltwo{w_{t+1}-w(\theta_{t+1})}^2 \right] &\leq \left(1 - \Theta(\lambda_v\beta)  \right)\mE\left[ \ltwo{w_{t}-w(\theta_{t})}^2 \right] \nonumber\\
	&\quad + \Theta(\alpha^2/\beta)\mE\left[\ltwo{\nabla J(\theta_t)}^2\right] + \Theta(1/M).
	\end{flalign}
	
	\noindent\textbf{Step 3:} We then establish the induction relationships for the gradient norm $\ltwo{\nabla J(\theta_t)}^2$:
	\begin{flalign}\label{eq: sketch5}
	\left(\Theta(\alpha) - \Theta(\alpha^2) \right) \mE\left[\ltwo{\nabla J(\theta_t)}^2\right]&\leq \mE[J(\theta_t)] - \mE[J(\theta_{t+1})] + \Theta(\alpha + \alpha^2)\mE\left[ \ltwo{w_{t}-w(\theta_{t})}^2 \right] +  \Theta(1/M).
	\end{flalign}
	
	\noindent\textbf{Step 4:} Applying \cref{eq: sketch4} and \cref{eq: sketch5} recursively from $t=T-1$ to $0$ and combing those two results together yield
	\begin{flalign*}
	\left( \Theta(\alpha) - \Theta(\alpha^2) - \Theta(\alpha^3) \right)\sum_{t=0}^{T-1}\mE\left[\ltwo{\nabla J(\theta_t)}^2\right]\leq J(\theta_0) - \mE[J(\theta_T)] + \Theta((\alpha + \alpha^2)/\beta)\ltwo{w_0 - w(\theta_0)}^2 + \Theta(1/M).
	\end{flalign*}
	Letting the stepsize $\alpha$ and $\beta$ and the batch size $M$ satisfies the requirement specified in \Cref{thm2}, we can then obtain the desired convergence result.
\end{proof}

\section{Policy Optimization: Greedy-GQ Algorithm}

In this section, we will provide the non-asymptotic convergence result of Greedy-GQ \cite{maei2010toward}, which is also a two timescale nonlinear SA algorithm. 

Greedy-GQ was proposed in \cite{maei2010toward} to solve the divergence issue of Q-Learning in the linear function approximation setting.
In Greedy-GQ, the goal of the agent is to learn an optimal policy for the MDP with respect to the total expected discounted reward.
In the linear function approximation setting, a linear function $\hat{Q}(s,a,\theta)=\phi(s,a)^\top \theta$ is used to approximate the state-action value function $Q(s,a)$, where $\phi(s,a)\in \mR^d$ is a fixed feature vector for state-action pair $(s,a)$ and $\theta\in \mR^d$ is a parameter vector. Without loss of generality, we assume that the feature vector $\ltwo{\phi(s,a)}\leq 1$ for all $(s,a)\in\mcs\times \mca$ and the columns of the feature matrix $\Phi$ are linearly independent.  In this setting, we hope to find a solution $\theta$ that satisfies 
\begin{flalign}\label{eq: gq_goal}
\Pi T^{\pi_{\theta}}\hat{Q}(s,a,\theta)=\hat{Q}(s,a,\theta),\quad \text{for all } (s,a)\in \mcs\times\mca,
\end{flalign}
where $\pi_{\theta}$ is the soft-max greedy policy with respect to the state-action value function $\hat{Q}(s,a,\theta)$, i.e., $\pi_\theta(a|s)=\frac{\exp(\tau\hat{Q}(s,a,\theta))}{\sum_{a^\prime \in \mca}\exp(\tau\hat{Q}(s,a^\prime,\theta))}$, where $\tau>0$ is the temperature parameter, and $T^{\pi_{\theta}}$ denotes the Bellman operator with policy $\pi_{\theta}$.
Similarly to the TDC algorithms, Gready-GQ searches a parameter that satisfies \cref{eq: gq_goal} by minimizing a projected Bellman error objective function defined as:
\begin{flalign}\label{mspbe2}
J(\theta)=\mE_{\mu_{\pi_b}}[\hat{Q}(s,a,\theta)-{\rm \Pi} T^{\pi_\theta} \hat{Q}(s,a,\theta)]^2,
\end{flalign}
where $\delta(\theta)=r(s,a,s^\prime) + \gamma \hat{Q}(s^\prime,b,\theta)-\hat{Q}(s,a,\theta)$ is the temporal difference error, with $a \sim \pi_\theta(\cdot|s)$ and $b \sim \pi_\theta(\cdot|s^\prime)$. Since $J(\theta)$ is nonconvex and smooth everywhere, we can apply gradient descent method to find a local optimal (stationary point) of the objective $J(\theta)$ via applying the update $\theta_{t+1}=\theta_t - \frac{\alpha}{2}\nabla J(\theta_t)$ iteratively, in which
\begin{flalign*}
\frac{1}{2}\nabla J(\theta_t)&=-\mE_{\mu_{\pi_b}}[\mE_{\pi_\theta}[\delta(\theta)|s,a]\phi(s,a)] + \gamma \mE_{\mu_{\pi_b}}[\mE_{\pi_\theta}[\phi(s^\prime,b)|s,a]\phi(s,a)^\top] w(\theta)
\end{flalign*}
where
{\small \begin{flalign*}
	w(\theta)=\mE_{\mu_{\pi_b}}[\phi(s,a)\phi(s,a)^\top]^{-1}\mE_{\mu_{\pi_b}}[\mE_{\pi_\theta}[\delta(\theta)|s,a]\phi(s,a)].
	\end{flalign*} }
Similarly to the nonlinear TDC algorithms in \cref{sc: nonlinearTDC}, here an auxiliary parameter $w_t$ is adopted to estimate the vector $w(\theta_t)$ by solving a linear SA with the following corresponding ODE:
{\small \begin{flalign*}
	\dot{w}=-\mE_{\mu_{\pi_b}}[\phi(s,a)\phi(s,a)^\top]w + \mE_{\mu_{\pi_b}}[\mE_{\pi_\theta}[\delta(\theta)|s,a]\phi(s,a)].
	\end{flalign*}}
Then, $\theta_t$ can be updated via direct sampling
\begin{flalign*}
\theta_{t+1}=\theta_t + \alpha_t \frac{1}{\lone{\mB_t}}\sum_{j\in \mB_t} g(\theta_t, w_t,x_j),
\end{flalign*}
where $\rho_{\theta}(s,a)=\pi_\theta(a|s)/\pi_b(a|s)$ is the importance weighting factor bounded by $\rho_{\max}$ and we define $g(\theta_t, w_t,x_j)=\rho_{\theta_t}(s_{j+1}, a_{j+1})(\delta_j(\theta_t)\phi(s_j,a_j) - \gamma \phi(s_{j+1},a_{j+1}) \phi(s_j,a_j)^\top w_t)$. The two timescale Greedy-GQ algorithm is shown below. 

\begin{algorithm}[H]
	\null
	\caption{Two Timescale Greedy-GQ}
	\label{al: GreedyGQ}
	\begin{algorithmic}[1]
		\STATE {\bfseries Input:} batch size $M$, learning rate $\alpha$ and $\beta$
		\STATE {\bfseries Sampling:} A trajectory $\{ s_j, a_j\}_{j\geq 0}$ is sampled by following the behaviour policy $\pi_b$
		\STATE \textbf{Initialization: } $\theta_0$ and $w_0$
		\FOR{$t=0,\cdots,T-1$}
		\STATE $i_t=tM$
		\STATE {\small $w_{t+1} = w_t + \beta\frac{1}{M}\sum_{j=i_t}^{i_t+M-1}(-\phi(s_j, a_j)\phi(s_j, a_j)^\top w_t + \rho_{\theta_j}(s_j, a_j)\delta_j(\theta_t)\phi(s_j, a_j))$}
		\STATE {\small $\theta_{t+1}=\theta_t + \alpha\frac{1}{M}\sum_{j=i_t}^{i_t+M-1} \rho_{\theta_t}(s_{j+1}, a_{j+1})(\delta_j(\theta_t)\phi(s_j)-\gamma \phi(s_{j+1}, a_{j+1})\phi(s_j, a_j)^\top w_t)$}
		\ENDFOR
		\OUTPUT $\tilde{\theta}_{\hat{T}}$ with $\hat{T}$ chosen uniformly from $\{1,\cdots,T\}$
	\end{algorithmic}
\end{algorithm}

By slightly abusing notations in \Cref{sc: nonlinearTDC}, we make the follow standard assumptions.
\begin{assumption}[Non-singularity]\label{ass6}
	We have $(\max_{\theta\in \mR^d}\lone{\lambda_{\max}\{A^\top_\theta C^{-1} A_\theta\}})^{-1}=\lambda_1$ and $\lone{\lambda_{\max}\{C\}} = \lambda_2$, where $A_\theta=\mE_{\mu_{\pi_b}}[(\gamma\mE_{\pi_\theta}[\phi(s^\prime)|s]-\phi(s))\phi(s)^\top]$ and $C=-\mE_{\mu_{\pi_b}}[\phi(s)\phi(s)^\top]$ and $\lambda_1$ and $\lambda_2$ are positive constants.
\end{assumption}
\begin{assumption}[Bounded importance factor]\label{ass7}
	For any state-action pair $(s,a)\in \ms\times\ma$ and any $\theta\in \mR^d$, we have $\rho_\theta(s,a)\leq \rho_{\max}$, where $\rho_{\max}$ is a positive constant.
\end{assumption}
Note that \Cref{ass7} can be satisfied when the behaviour policy is non-degenerated for all states.
Moreover, \cite{wang2020finite} provide the following Lipschitz property of the gradient $\nabla J(\theta)$.
\begin{lemma}
	Suppose \Cref{ass1} and \Cref{ass6} hold, for any $\theta, \theta^\prime\in \mR^d$, we have $\ltwo{\nabla J(\theta)-\nabla J(\theta^\prime)}\leq L_J\ltwo{\theta-\theta^\prime}$, where $L_J$ is a positive constant.
\end{lemma}
Note that the Greedy-GQ algorithm in \Cref{al: GreedyGQ} and nonlinear TDC algorithm in \Cref{al: nonlinearTDC} share similar structures. Both objectives are nonconvex and both algorithms adopt a two timescale update scheme, in which the fast timescale iteration corresponds to a linear SA and the slow time-scale iteration corresponds to a nonlinear SA.
Thus, the analysis of two time-scale nonlinear TDC in \Cref{sc: nonlinearTDC} can be extended to study the convergence rate of Greedy-GQ algorithm. The following theorem characterizes the convergence rate and sample complexity of \Cref{al: GreedyGQ}.
\begin{theorem}\label{thm3}
	Consider the two timescale Greedy-GQ algorithm in \Cref{al: GreedyGQ}. Suppose Assumptions \ref{ass1}, \ref{ass6} and \ref{ass7} hold. Let the stepsize $\beta\leq \min\{ \frac{\lambda_2}{4}, \frac{8}{\lambda_2} \}$ and $\alpha\leq \min\{ \frac{1}{8L_J},\frac{\lambda_2\sqrt{\lambda_2}}{8\sqrt{2}\rho_{\max}}\beta, \frac{L_J\lambda^3_2}{5312\rho^2_{\max}\lambda^2_1}\beta^2 \}$, and batch size $M\geq \frac{1+(\kappa-1)\rho}{1-\rho} \max\{128\left( \rho^2_{\max} + \frac{1}{\lambda^2_2} \right)[1+ \frac{\lambda^2_2\beta}{4\alpha^2}(\frac{2\beta}{\lambda_2} + 2\beta^2)], \frac{\beta^2\lambda^3_2(\rho_{\max}+1)^4}{\rho^2_{\max}\alpha^2}\}$. We have
	{\begin{flalign*}
		\mE\left[\ltwo{\nabla J(\theta_{\hat{T}})}^2\right]&\leq \frac{8(J(\theta_0)-\mE[J(\theta_T)])}{\alpha T} + \frac{192\rho^2_{\max}}{\lambda_2\beta}\frac{\ltwo{w_0 - w^*(\theta_0)}^2}{T} + \frac{32C_1[1+(\kappa-1)\rho]}{M(1-\rho)},
		\end{flalign*}}
	where $C_1$ is a positive constant defined in \cref{eq: 30} in \Cref{sc: appGQ}. Furthermore, let $M\geq \frac{64C_2[1+(\kappa-1)\rho]}{(1-\rho)\epsilon}$ and $T\geq \frac{2}{\epsilon}\left[ \frac{8J(\theta_0)}{\alpha} + \frac{192\rho^2_{\max}\ltwo{w_0-w(\theta_0)}^2}{\lambda_2\beta} \right]$. The total sample complexity for \Cref{al: nonlinearTDC} to achieve an $\epsilon$-accurate stationary point, i.e., $\mE\big[ \ltwo{\nabla J(\theta_{\hat{T}})}^2 \big]\leq \epsilon$, is given by
	\begin{flalign*}
	TM=\Theta\left( \frac{1}{\epsilon^2} \right).
	\end{flalign*}
\end{theorem}

Similarly to \Cref{thm2}, in \Cref{thm3} we show that \Cref{al: GreedyGQ} converges to an $\epsilon$-accurate stationary point with sample complexity $\mathcal{O}(\epsilon^{-2})$. Note that \cite{wang2020finite} studied the convergence rate of two timescale Greedy-GQ with diminishing stepsize, which achieves the complexity of $\mathcal{O}(\epsilon^{-3}\log(1/\epsilon))$. \Cref{thm3} for two timescale Greedy-GQ with constant stepsize outperforms the result in \cite{wang2020finite} by a factor of $\mathcal{O}(\epsilon^{-1}\log(1/\epsilon))$, indicating that the constant stepsize can significantly improve the convergence rate of two timescale Greedy-GQ algorithm.

\section{Conclusion}

In this paper, we study the convergence rate for two timescale linear and nonlinear TDC and Greedy-GQ under Markovian sampling and constant stepsize. Specifically, we show that the complexity result of linear TDC orderwisely achieves the optimal convergence rate under a constant stepsize. Our result for nonlinear TDC is the first under Markovian sampling. Moreover, our sample complexity result of Greedy-GQ outperforms the previous result orderwisely.
For future work, it is interesting to apply more advance optimization techniques, e.g., acceleration, variance reduction, to further improve the convergence performance of the value-based RL algorithms studied in this paper.

\newpage
\bibliographystyle{apalike}
\bibliography{ref}

\newpage
\appendix
\noindent {\Large \textbf{Supplementary Materials}}

\section{Convergence Analysis of Two Time-scale Linear TDC}\label{sc: proflemma1}
We first provide the following lemma that is useful for the proof of \Cref{thm1}, which is proved in \cite{xu2020improving}. Throughout the paper, for two matrices $M, N\in \mR^{d\times d}$, we define $\langle M, N \rangle = \sum_{i=1}^{d}\sum_{j=1}^{d} M_{i,j}N_{i,j}$.
\begin{lemma}\label{lemma1}
	Consider a sequence $\{ s_t \}_{t\geq 0}$ generated by the MDP defined in \Cref{sc: mdp}. Suppose \Cref{ass2} holds. Let $X(s)$ be either a matrix or a vector that satisfies the following conditions:
	\begin{flalign*}
	\ltwo{X(s)}\,(\text{vector})\,\text{or}\,\lF{X(S)}\,(\text{matrix}) \leq C_x \quad \text{for all} \quad s\in \mcs,
	\end{flalign*}
	and
	\begin{flalign*}
	\mE_{\nu}[X(s)] = \widetilde{X}.
	\end{flalign*}
	For any $t_0\geq 0$ and $M>0$, define $X(\mam)=\frac{1}{M}\sum_{i=t_0}^{t_0+M-1}X(s_i)$. We have
	\begin{flalign*}
	\mE\left[\ltwo{X(\mam)-\widetilde{X}}^2\right]\leq \frac{8C_x^2[1+(\kappa-1)\rho]}{(1-\rho)M}.
	\end{flalign*}
\end{lemma}

We next proceed to prove \Cref{thm1}.
\begin{proof}[Proof of Theorem \ref{thm1}]
	We define $w^*(\theta)=-C^{-1}(A\theta+b)$, $\theta^*=-A^{-1}b$, and 
	\begin{flalign}
	g(\theta_t)&= (A-BC^{-1}A)\theta_t + (b - BC^{-1}b), \label{tdc: tdc_g}\\
	f(w_t)&= C(w_t-w^*(\theta_t)), \label{tdc_f}
	\end{flalign}
	where $B=-\gamma\mE[\mE_\pi[\phi(s^\prime)|s]\phi(s)^\top]$. We further define
	\begin{flalign}
	g_t(\theta_t)&= (A_t-B_tC^{-1}A)\theta_t + (b_t - B_tC^{-1}b), \label{tdc_gt}\\
	f_t(w_t)&= C_t(w_t-w^*(\theta_t)), \label{tdc_ft}\\
	h_t(\theta_t)&=(A_t-C_tC^{-1}A)\theta_t + (b_t - C_tC^{-1}b), \label{tdc_ht}
	\end{flalign}
	where $A_t=(\gamma\rho(s_t,a_t)\phi(s_{t+1})-\phi(s_t))\phi(s_t)$, $B_t=-\gamma\rho(s_t,a_t)\phi(s_{t+1})\phi(s_t)^\top$, $C_t=-\phi(s_t)\phi(s_t)^\top$ and $b_t=\rho(s_t,a_t)r(s_t,a_t,s_{t+1})\phi(s_t)$.
	The update of two time-scale linear TDC (line 5-6 of \Cref{al: linearTDC}) can be rewritten as
	\begin{flalign}
	\theta_{t+1}&=\theta_t + \alpha [g_t(\theta_t) + B_t(w_t - w^*(\theta_t))], \label{tdc_theta}\\
	w_{t+1} &= w_t + \beta [f_t(w_t) + h_t(\theta_t)]. \label{tdc_w}
	\end{flalign}
	Considering the iteration of $w_t$, we proceed as follows:
	\begin{flalign}
	&\ltwo{w_{t+1}-w^*(\theta_{t})}^2\nonumber\\
	&=\ltwo{w_t + \beta [f_t(w_t) + h_t(\theta_t)] - w^*(\theta_{t})}^2\nonumber\\
	&=\ltwo{w_t - w^*(\theta_t)}^2 + 2\beta\langle w_t - w^*(\theta_t), f_t(w_t)  \rangle + 2\beta\langle w_t - w^*(\theta_t), h_t(w_t)  \rangle \nonumber\\
	&\quad + \beta^2\ltwo{f_t(w_t) + h_t(\theta_t)}^2\nonumber\\
	&=\ltwo{w_t - w^*(\theta_t)}^2 + 2\beta\langle w_t - w^*(\theta_t), f(w_t)  \rangle + 2\beta\langle w_t - w^*(\theta_t), f_t(w_t) -f(w_t)  \rangle  \nonumber\\
	&\quad + 2\beta\langle w_t - w^*(\theta_t), h_t(w_t)  \rangle + \beta^2\ltwo{f_t(w_t) + h_t(\theta_t)}^2\nonumber\\
	&\overset{(i)}{\leq}(1-2\lambda_2\beta)\ltwo{w_t - w^*(\theta_t)}^2 + 2\beta\left[ \frac{\lambda_2}{4}\ltwo{w_t - w^*(\theta_t)}^2 + \frac{1}{\lambda_2}\ltwo{f_t(w_t) -f(w_t)}^2 \right] \nonumber\\
	&\quad + 2\beta \left[ \frac{\lambda_2}{4}\ltwo{w_t - w^*(\theta_t)}^2 + \frac{1}{\lambda_2}\ltwo{h_t(\theta_t)}^2 \right] + 2\beta^2\ltwo{f_t(w_t)}^2 + 2\beta^2\ltwo{h_t(\theta_t)}^2\nonumber\\
	&\overset{(ii)}{\leq} (1-\lambda_2\beta + 2\beta^2)\ltwo{w_t - w^*(\theta_t)}^2 + \frac{2\beta}{\lambda_2}\ltwo{f_t(w_t) -f(w_t)}^2 + \left( \frac{2\beta}{\lambda_2} + 2\beta^2 \right)\ltwo{h_t(\theta_t)}^2,\label{eq: 2}
	\end{flalign}
	where $(i)$ follows from the fact that $\langle w_t - w^*(\theta_t), f(w_t)  \rangle = \langle w_t - w^*(\theta_t), C(w_t - w^*(\theta_t))  \rangle \leq -\lambda_2\ltwo{w_t - w^*(\theta_t)}^2$ and Young's inequality, $(ii)$ follows from the fact that $\ltwo{f_t(w_t)}=\ltwo{C_t(w_t - w^*(\theta_t))}\leq \ltwo{C_t} \ltwo{w_t - w^*(\theta_t)}\leq \ltwo{w_t - w^*(\theta_t)}$. Taking expectation conditioned on $\mf_{t}$ on both sides of \cref{eq: 2} yields
	\begin{flalign}
	&\mE[\ltwo{w_{t+1}-w^*(\theta_{t})}^2|\mf_t] \nonumber\\
	&\leq (1-\lambda_2\beta + 2\beta^2)\ltwo{w_t - w^*(\theta_t)}^2 + \frac{2\beta}{\lambda_2}\mE[\ltwo{f_t(w_t) -f(w_t)}^2|\mf_t] \nonumber\\
	&\quad + \left( \frac{2\beta}{\lambda_2} + 2\beta^2 \right)\mE[\ltwo{h_t(\theta_t)}^2|\mf_t]\nonumber\\
	&\overset{(i)}{\leq} \left[1-\lambda_2\beta + 2\beta^2 + \frac{16\beta}{\lambda_2}\frac{1+(\kappa-1)\rho}{(1-\rho)M} \right]\ltwo{w_t - w^*(\theta_t)}^2\nonumber\\
	&\quad + 128\left(1+ \frac{1}{\lambda^2_2} \right)\left( \frac{2\beta}{\lambda_2} + 2\beta^2 \right) \frac{1+(\kappa-1)\rho}{(1-\rho)M}\ltwo{\theta_t-\theta^*}^2 \nonumber\\
	&\quad + 32(4R^2_\theta + r^2_{\max})\left( \frac{2\beta}{\lambda_2} + 2\beta^2 \right)\frac{1+(\kappa-1)\rho}{(1-\rho)M}\nonumber\\
	&\overset{(ii)}{\leq} \left( 1- \frac{\lambda_2\beta}{2} \right)\ltwo{w_t - w^*(\theta_t)}^2 + 128\left(\rho^2_{\max}+ \frac{1}{\lambda^2_2} \right)\left( \frac{2\beta}{\lambda_2} + 2\beta^2 \right) \frac{1+(\kappa-1)\rho}{(1-\rho)M}\ltwo{\theta_t-\theta^*}^2\nonumber\\
	&\quad + 32(4R^2_\theta\rho^2_{\max} + r^2_{\max})\left( \frac{2\beta}{\lambda_2} + 2\beta^2 \right)\frac{1+(\kappa-1)\rho}{(1-\rho)M}, \label{eq: 7}
	\end{flalign}
	where $(i)$ follows from the facts that
	\begin{flalign*}
	\mE[\ltwo{f_t(w_t) -f(w_t)}^2|\mf_t] &= \mE[\ltwo{(C_t-C)(w_t-w^*(\theta_t))}^2|\mf_t] \nonumber\\
	&\leq \mE[\ltwo{(C_t-C)}^2|\mf_t] \ltwo{w_t-w^*(\theta_t)}^2 \nonumber\\
	&\overset{(a)}{\leq} \frac{8[1+(\kappa-1)\rho]}{(1-\rho)M}\ltwo{w_t-w^*(\theta_t)}^2,
	\end{flalign*}
	and
	\begin{flalign*}
	&\mE[\ltwo{h_t(\theta_t)}^2|\mf_t] \nonumber\\
	&= \mE[\ltwo{(A_t-C_tC^{-1}A)\theta_t + (b_t - C_tC^{-1}b)}^2|\mf_t]\nonumber\\
	&= \mE[\ltwo{(A_t-A)(\theta_t-\theta^*) + (A_t-A)\theta^* + b_t-b + (C-C_t)C^{-1}A(\theta_t-\theta^*) }^2|\mf_t]\nonumber\\
	&\leq 4\mE[\ltwo{(A_t-A)(\theta_t-\theta^*)}^2|\mf_t] + 4\mE[\ltwo{(A_t-A)\theta^*}^2|\mf_t] + 4\mE[\ltwo{b_t-b}^2|\mf_t] \nonumber\\
	&\quad + 4\mE[\ltwo{(C-C_t)C^{-1}A(\theta_t-\theta^*) }^2|\mf_t]\nonumber\\
	&\leq 4\mE[\ltwo{A_t-A}^2|\mf_t] \ltwo{\theta_t-\theta^*}^2 + 4\mE[\ltwo{A_t-A}^2|\mf_t] \ltwo{\theta^*}^2 + 4\mE[\ltwo{b_t-b}^2|\mf_t] \nonumber\\
	&\quad + 4\mE[\ltwo{(C-C_t) }^2|\mf_t] \ltwo{C^{-1}}^2 \ltwo{A}^2 \ltwo{\theta_t-\theta^*}^2\nonumber\\
	&\overset{(b)}{\leq} 128\left(\rho^2_{\max}+ \frac{1}{\lambda^2_2} \right)\frac{1+(\kappa-1)\rho}{(1-\rho)M}\ltwo{\theta_t-\theta^*}^2 + 32(4R^2_\theta \rho^2_{\max} + r^2_{\max})\frac{1+(\kappa-1)\rho}{(1-\rho)M},
	\end{flalign*}
	where $(a)$ and $(b)$ follow from \Cref{lemma1} and the fact that $\ltwo{\theta^*}\leq R_\theta$, where $R_{\theta} = \frac{r_{\max}}{\lambda_1}$, and $(ii)$ follows from the fact that $\beta\leq \frac{\lambda_2}{4}$ and $M\geq \frac{64[1+(\kappa-1)\rho]}{\lambda_2^2(1-\rho)}$. Then, we upper bound the term $\mE[\ltwo{w_{t+1}-w^*(\theta_{t+1})}^2|\mf_t]$ as follows:
	\begin{flalign}
	&\mE[\ltwo{w_{t+1}-w^*(\theta_{t+1})}^2|\mf_t]\nonumber\\
	&\leq \left(1 + \frac{1}{2(2/(\lambda_2\beta)-1)}\right)\mE[\ltwo{w_{t+1}-w^*(\theta_{t})}^2] + (1 + 2(2/(\lambda_2\beta)-1))\mE[\ltwo{w^*(\theta_{t+1})-w^*(\theta_t)}^2]\nonumber\\
	&\overset{(i)}{\leq} \left( \frac{4/(\lambda_2\beta)-1}{4/(\lambda_2\beta)-2} \right)\left( 1- \frac{\lambda_2\beta}{2} \right)\ltwo{w_t - w^*(\theta_t)}^2 + \frac{8}{\lambda^2_2 \beta}\mE[\ltwo{\theta_{t+1}-\theta_t}^2] \nonumber\\
	&\quad + 128\left( \frac{4/(\lambda_2\beta)-1}{4/(\lambda_2\beta)-2} \right)\left(\rho^2_{\max}+ \frac{1}{\lambda^2_2} \right)\left( \frac{2\beta}{\lambda_2} + 2\beta^2 \right) \frac{1+(\kappa-1)\rho}{(1-\rho)M}\ltwo{\theta_t-\theta^*}^2\nonumber\\
	&\quad + 32\left( \frac{4/(\lambda_2\beta)-1}{4/(\lambda_2\beta)-2} \right)(4R^2_\theta + r^2_{\max})\left( \frac{2\beta}{\lambda_2} + 2\beta^2 \right)\frac{1+(\kappa-1)\rho}{(1-\rho)M}\nonumber\\
	&\overset{(ii)}{\leq} \left( 1- \frac{\lambda_2\beta}{4} \right)\ltwo{w_t - w^*(\theta_t)}^2 + \frac{8}{\lambda^2_2 \beta}\mE[\ltwo{\theta_{t+1}-\theta_t}^2] \nonumber\\
	&\quad + 128\left(\rho^2_{\max}+ \frac{1}{\lambda^2_2} \right)\left( \frac{2\beta}{\lambda_2} + 2\beta^2 \right) \frac{1+(\kappa-1)\rho}{(1-\rho)M}\ltwo{\theta_t-\theta^*}^2\nonumber\\
	&\quad + 32(4R^2_\theta \rho^2_{\max} + r^2_{\max})\left( \frac{2\beta}{\lambda_2} + 2\beta^2 \right)\frac{1+(\kappa-1)\rho}{(1-\rho)M}\nonumber\\
	&\leq \left( 1- \frac{\lambda_2\beta}{4} \right)\ltwo{w_t - w^*(\theta_t)}^2 + \frac{16\alpha^2}{\lambda^2_2 \beta}\mE[\ltwo{B_t(w_t - w^*(\theta_t))}^2] + \frac{16\alpha^2}{\lambda^2_2 \beta}\mE[\ltwo{g_t(\theta_t)}^2] \nonumber\\
	&\quad + 128\left(\rho^2_{\max}+ \frac{1}{\lambda^2_2} \right)\left( \frac{2\beta}{\lambda_2} + 2\beta^2 \right) \frac{1+(\kappa-1)\rho}{(1-\rho)M}\ltwo{\theta_t-\theta^*}^2\nonumber\\
	&\quad + 32(4R^2_\theta\rho^2_{\max} + r^2_{\max})\left( \frac{2\beta}{\lambda_2} + 2\beta^2 \right)\frac{1+(\kappa-1)\rho}{(1-\rho)M}\nonumber\\
	&\leq \left( 1- \frac{\lambda_2\beta}{4} + \frac{16\rho^2_{\max}\alpha^2}{\lambda^2_2 \beta} \right)\ltwo{w_t - w^*(\theta_t)}^2 + \frac{32\alpha^2}{\lambda^2_2 \beta}\mE[\ltwo{g_t(\theta_t)-g(\theta_t)}^2] \nonumber\\
	&\quad + \frac{32\alpha^2}{\lambda^2_2 \beta}\mE[\ltwo{g(\theta_t)}^2] + 128\left(\rho^2_{\max}+ \frac{1}{\lambda^2_2} \right)\left( \frac{2\beta}{\lambda_2} + 2\beta^2 \right) \frac{1+(\kappa-1)\rho}{(1-\rho)M}\ltwo{\theta_t-\theta^*}^2\nonumber\\
	&\quad + 32(4R^2_\theta\rho^2_{\max} + r^2_{\max})\left( \frac{2\beta}{\lambda_2} + 2\beta^2 \right)\frac{1+(\kappa-1)\rho}{(1-\rho)M}\nonumber\\
	&\overset{(iii)}{\leq} \left( 1- \frac{\lambda_2\beta}{4} + \frac{16\rho^2_{\max}\alpha^2}{\lambda^2_2 \beta} \right)\ltwo{w_t - w^*(\theta_t)}^2 \nonumber\\
	&\quad + \frac{32\alpha^2}{\lambda^2_2 \beta}\left[ 128\left( \rho^2_{\max} + \frac{1}{\lambda^2_2} \right)\frac{1+(\kappa-1)\rho}{(1-\rho)M}\ltwo{\theta_t-\theta^*}^2 + 32(4R^2_\theta \rho^2_{\max} + r^2_{\max})\frac{1+(\kappa-1)\rho}{(1-\rho)M} \right] \nonumber\\
	&\quad + \frac{64\alpha^2}{\lambda^2_2 \beta} \ltwo{\theta_t-\theta^*}^2 + 128\left(\rho^2_{\max}+ \frac{1}{\lambda^2_2} \right)\left( \frac{2\beta}{\lambda_2} + 2\beta^2 \right) \frac{1+(\kappa-1)\rho}{(1-\rho)M}\ltwo{\theta_t-\theta^*}^2\nonumber\\
	&\quad + 32(4R^2_\theta\rho^2_{\max} + r^2_{\max})\left( \frac{2\beta}{\lambda_2} + 2\beta^2 \right)\frac{1+(\kappa-1)\rho}{(1-\rho)M}\nonumber\\
	&\leq \left( 1- \frac{\lambda_2\beta}{4} + \frac{16\rho^2_{\max}\alpha^2}{\lambda^2_2 \beta} \right)\ltwo{w_t - w^*(\theta_t)}^2 +  \left(\frac{96\alpha^2}{\lambda^2_2 \beta} + \frac{\lambda_1\alpha}{4} \right) \ltwo{\theta_t-\theta^*}^2 \nonumber\\
	&\quad + 32(4R^2_\theta\rho^2_{\max} + r^2_{\max})\left(  \frac{32\alpha^2}{\lambda^2_2 \beta} + \frac{2\beta}{\lambda_2} + 2\beta^2 \right)\frac{1+(\kappa-1)\rho}{(1-\rho)M},
	\end{flalign}
	where $(i)$ follows Yong's inequality, $(ii)$ follows from the fact that $\beta \leq  \min\left\{ \frac{1}{8\lambda_2}, \frac{\lambda_2}{4}\right\}$, and $(iii)$ follows from the fact that
	\begin{flalign}
	&\mE[\ltwo{g_t(\theta_t)-g(\theta_t)}^2|\mf_t] \nonumber\\
	&=\mE[\ltwo{ (A_t-A)(\theta_t-\theta^*) + (A_t-A)\theta^* + (b_t-b) + (B-B_t)C^{-1}A(\theta_t-\theta^*) }^2|\mf_t] \nonumber\\
	&\leq 4\mE[\ltwo{A_t-A}^2|\mf_t]\ltwo{\theta_t-\theta^*}^2 + 4\mE[\ltwo{A_t-A}^2|\mf_t]\ltwo{\theta^*}^2 + 4\mE[\ltwo{b_t-b}^2|\mf_t] \nonumber\\
	&\quad + 4\mE[\ltwo{B-B_t }^2|\mf_t]\ltwo{C^{-1}}^2\ltwo{A}^2\ltwo{\theta_t-\theta^*}^2 \nonumber\\
	&\overset{(a)}{\leq} 128\left( \rho^2_{\max} + \frac{1}{\lambda^2_2} \right)\frac{1+(\kappa-1)\rho}{(1-\rho)M}\ltwo{\theta_t-\theta^*}^2 + 32(4R^2_\theta \rho^2_{\max} + r^2_{\max})\frac{1+(\kappa-1)\rho}{(1-\rho)M},\label{eq: 31}
	\end{flalign}
	where $(a)$ follows from \Cref{lemma1}, and $(ii)$ follows from the fact that $M\geq 128\left( \rho^2_{\max} + \frac{1}{\lambda^2_2} \right)\frac{1+(\kappa-1)\rho}{1-\rho}\max\{1,\frac{8\beta+8\lambda_2\beta^2}{\lambda_1\lambda_2\alpha}\}$.
	Considering the iterate of $\theta_t$, we proceed as follows:
	\begin{flalign}
	&\ltwo{\theta_{t+1}-\theta^*}^2\nonumber\\
	&=\ltwo{\theta_t + \alpha [g_t(\theta_t) + B_t(w_t - w^*(\theta_t))] - \theta^*}^2\nonumber\\
	&=\ltwo{\theta_t-\theta^*}^2 + 2\alpha\langle \theta_t-\theta^*, g_t(\theta_t) \rangle + 2\alpha\langle \theta_t-\theta^*, B_t(w_t - w^*(\theta_t)) \rangle \nonumber\\
	&\quad + \alpha^2 \ltwo{g_t(\theta_t) + B_t(w_t - w^*(\theta_t))}^2 \nonumber\\
	&= \ltwo{\theta_t-\theta^*}^2 + 2\alpha\langle \theta_t-\theta^*, g(\theta_t) \rangle + 2\alpha\langle \theta_t-\theta^*, g_t(\theta_t) - g(\theta_t) \rangle  \nonumber\\
	&\quad + 2\alpha\langle \theta_t-\theta^*, B_t(w_t - w^*(\theta_t)) \rangle + \alpha^2 \ltwo{g_t(\theta_t) + B_t(w_t - w^*(\theta_t))}^2 \nonumber\\
	&\overset{(i)}{\leq} (1 - 2\lambda_1\alpha)\ltwo{\theta_t-\theta^*}^2 + 2\alpha\left[ \frac{\lambda_1}{4}\ltwo{\theta_t-\theta^*}^2 +\frac{1}{\lambda_1} \ltwo{g_t(\theta_t) - g(\theta_t)}^2 \right] \nonumber\\
	&\quad + 2\alpha\left[ \frac{\lambda_1}{4}\ltwo{\theta_t-\theta^*}^2 + \frac{1}{\lambda_1}\ltwo{B_t(w_t - w^*(\theta_t))}^2 \right] + 3\alpha^2 \ltwo{g(\theta_t)}^2 \nonumber\\
	&\quad + 3\alpha^2\ltwo{g_t(\theta_t)-g(\theta_t)}^2 + 3\alpha^2\ltwo{B_t(w_t - w^*(\theta_t))}^2\nonumber\\
	&\overset{(ii)}{\leq} (1-\lambda_1\alpha)\ltwo{\theta_t-\theta^*}^2 + \left( \frac{2\alpha}{\lambda_1} + 3\alpha^2 \right) \ltwo{g_t(\theta_t)-g(\theta_t)}^2 \nonumber\\
	&\quad + \rho_{\max}^2\left( \frac{2\alpha}{\lambda_1} + 3\alpha^2 \right) \ltwo{w_t - w^*(\theta_t)}^2 + 3\alpha^2\ltwo{g(\theta_t)}^2\nonumber\\
	&\overset{(iii)}{\leq} \left(1-\lambda_1\alpha + \frac{3\alpha^2}{\lambda_2}\right)\ltwo{\theta_t-\theta^*}^2 + \left( \frac{2\alpha}{\lambda_1} + 3\alpha^2 \right) \ltwo{g_t(\theta_t)-g(\theta_t)}^2 \nonumber\\
	&\quad + \rho^2_{\max}\left( \frac{2\alpha}{\lambda_1} + 3\alpha^2 \right) \ltwo{w_t - w^*(\theta_t)}^2,\label{eq: 1}
	\end{flalign}
	where $(i)$ follows from the fact that $\langle \theta_t-\theta^*, g(\theta_t) \rangle = \langle \theta_t-\theta^*, A^\top C^{-1} A(\theta_t-\theta^*) \rangle \leq -\lambda_1\ltwo{\theta_t-\theta^*}$ and Young's inequality, $(ii)$ follows from the fact that $\ltwo{B_t(w_t - w^*(\theta_t))}\leq \ltwo{B_t}\ltwo{w_t - w^*(\theta_t)}\leq \rho_{\max}\ltwo{w_t - w^*(\theta_t)}$, and $(iii)$ follows from the fact that 
	\begin{flalign*}
	\ltwo{g(\theta_t)} = \ltwo{A^\top C^{-1} A(\theta_t-\theta^*)}\leq \ltwo{A^\top} \ltwo{C^{-1}} \ltwo{A}\ltwo{\theta_t-\theta^*}\leq \frac{1}{\lambda_2}\ltwo{\theta_t-\theta^*}.
	\end{flalign*}
	Taking expectation conditioned on $\mf_t$ on both sides of \cref{eq: 1} yields
	\begin{flalign}
	&\mE[\ltwo{\theta_{t+1}-\theta^*}^2|\mf_t]\nonumber\\
	&\leq \left(1-\lambda_1\alpha + \frac{3\alpha^2}{\lambda_2}\right)\ltwo{\theta_t-\theta^*}^2 + \left( \frac{2\alpha}{\lambda_1} + 3\alpha^2 \right) \mE[\ltwo{g_t(\theta_t)-g(\theta_t)}^2|\mf_t] \nonumber\\
	&\quad + \rho^2_{\max} \left( \frac{2\alpha}{\lambda_1} + 3\alpha^2 \right) \ltwo{w_t - w^*(\theta_t)}^2\nonumber\\
	&\overset{(i)}{\leq} \left[ 1-\lambda_1\alpha + \frac{3\alpha^2}{\lambda_2} + 128\left( \rho^2_{\max} + \frac{1}{\lambda^2_2} \right)\left( \frac{2\alpha}{\lambda_1} + 3\alpha^2 \right)\frac{1+(\kappa-1)\rho}{(1-\rho)M}  \right]\ltwo{\theta_t-\theta^*}^2 \nonumber\\
	&\quad + 32(4R^2_\theta \rho^2_{\max} + r^2_{\max})\left( \frac{2\alpha}{\lambda_1} + 3\alpha^2 \right)\frac{1+(\kappa-1)\rho}{(1-\rho)M} + \rho^2_{\max} \left( \frac{2\alpha}{\lambda_1} + 3\alpha^2 \right) \ltwo{w_t - w^*(\theta_t)}^2\nonumber\\
	&\overset{(ii)}{\leq} \left( 1- \frac{3}{4}\lambda_1\alpha + \frac{3\alpha^2}{\lambda_2}   \right) \ltwo{\theta_t-\theta^*}^2 + \rho^2_{\max} \left( \frac{2\alpha}{\lambda_1} + 3\alpha^2 \right) \ltwo{w_t - w^*(\theta_t)}^2 \nonumber\\
	&\quad + 32(4R^2_\theta \rho^2_{\max} + r^2_{\max})\left( \frac{2\alpha}{\lambda_1} + 3\alpha^2 \right)\frac{1+(\kappa-1)\rho}{(1-\rho)M} ,\label{eq: 4}
	\end{flalign}
	where $(i)$ follows from \cref{eq: 31} and $(ii)$ follows from the fact that $M\geq 128\left( \rho^2_{\max} + \frac{1}{\lambda^2_2} \right)\frac{1+(\kappa-1)\rho}{1-\rho}\max\{1,\frac{8+12\lambda_1\alpha}{\lambda_1}\}$.
	Combining \cref{eq: 7} and \cref{eq: 4} yields
	\begin{flalign*}
	&\mE[\ltwo{w_{t+1}-w^*(\theta_{t+1})}^2|\mf_t] + \mE[\ltwo{\theta_{t+1}-\theta^*}^2|\mf_t]\nonumber\\
	&\leq \left[ 1- \frac{\lambda_2\beta}{4} + \frac{16\rho^2_{\max}\alpha^2}{\lambda^2_2 \beta} + \rho^2_{\max} \left( \frac{2\alpha}{\lambda_1} + 3\alpha^2 \right) \right]\ltwo{w_t - w^*(\theta_t)}^2\nonumber\\
	&\quad + \left( 1- \frac{1}{2}\lambda_1\alpha + \frac{3\alpha^2}{\lambda_2} + \frac{96\alpha^2}{\lambda^2_2 \beta}  \right)\ltwo{\theta_t-\theta^*}^2\nonumber\\
	&\quad + 32(4R^2_\theta \rho^2_{\max} + r^2_{\max})\left( \frac{32\alpha^2}{\lambda^2_2 \beta} + \frac{2\beta}{\lambda_2} + 2\beta^2 +  \frac{2\alpha}{\lambda_1} + 3\alpha^2 \right)\frac{1+(\kappa-1)\rho}{(1-\rho)M}. 
	\end{flalign*}
	If we further let
	\begin{flalign*}
	&\alpha\leq \min\left\{ \frac{1}{8\lambda_1}, \frac{\lambda_1\lambda_2}{12}, \frac{\sqrt{\lambda_2\beta}}{4\sqrt{6}\rho_{\max}} \frac{\lambda_2\sqrt{\lambda_2}\beta}{16\rho^2_{\max}}, \frac{\lambda_1\lambda_2\beta}{64\rho^2_{\max}}, \frac{\lambda_1\lambda^2_2\beta}{768}  \right\},\quad \beta \leq \frac{1}{8\lambda_2}.
	\end{flalign*}
	We have
	\begin{flalign}
	&\mE[\ltwo{w_{t+1}-w^*(\theta_{t+1})}^2|\mf_t] + \mE[\ltwo{\theta_{t+1}-\theta^*}^2|\mf_t]\nonumber\\
	&\leq \left(1- \frac{1}{8}\min\{  \lambda_2\beta, \lambda_1\alpha \} \right)\left(\ltwo{w_t - w^*(\theta_t)}^2 + \ltwo{\theta_t - \theta^*}^2\right)\nonumber\\
	&\quad + 32(4R^2_\theta \rho^2_{\max} + r^2_{\max})\left( \frac{32\alpha^2}{\lambda^2_2 \beta} + \frac{2\beta}{\lambda_2} + 2\beta^2 +  \frac{2\alpha}{\lambda_1} + 3\alpha^2 \right)\frac{1+(\kappa-1)\rho}{(1-\rho)M}.\label{eq: 12}
	\end{flalign}
	Taking expectation on both sides of \cref{eq: 12} and applying the relation recursively from $t=T-1$ to $0$ yield
	\begin{flalign}
	&\mE[\ltwo{w_{T}-w^*(\theta_{T})}^2] + \mE[\ltwo{\theta_{T}-\theta^*}^2]\nonumber\\
	&\leq \left(1- \frac{1}{8}\min\{  \lambda_2\beta, \lambda_1\alpha \} \right)^{T}\left(\ltwo{w_0 - w^*(\theta_0)}^2 + \ltwo{\theta_0 - \theta^*}^2\right)\nonumber\\
	&\quad + 32(4R^2_\theta\rho^2_{\max} + r^2_{\max})\left( \frac{32\alpha^2}{\lambda^2_2 \beta} + \frac{2\beta}{\lambda_2} + 2\beta^2 +  \frac{2\alpha}{\lambda_1} + 3\alpha^2 \right)\frac{1+(\kappa-1)\rho}{(1-\rho)M}\sum_{t=0}^{T-1}\left(1- \frac{1}{8}\min\{  \lambda_2\beta, \lambda_1\alpha \} \right)^t\nonumber\\
	&\leq \left(1- \frac{1}{8}\min\{  \lambda_2\beta, \lambda_1\alpha \} \right)^{T}\left(\ltwo{w_0 - w^*(\theta_0)}^2 + \ltwo{\theta_0 - \theta^*}^2\right)\nonumber\\
	&\quad + \frac{256(4R^2_\theta\rho^2_{\max} + r^2_{\max})}{\min\{  \lambda_2\beta, \lambda_1\alpha \}}\left( \frac{32\alpha^2}{\lambda^2_2 \beta} + \frac{2\beta}{\lambda_2} + 2\beta^2 +  \frac{2\alpha}{\lambda_1} + 3\alpha^2 \right)\frac{1+(\kappa-1)\rho}{(1-\rho)M},
	\end{flalign}
	which implies
	\begin{flalign}
	\mE[\ltwo{\theta_{T}-\theta^*}^2]&\leq \left(1- \frac{1}{8}\min\{  \lambda_2\beta, \lambda_1\alpha \} \right)^{T}\left(\ltwo{w_0 - w^*(\theta_0)}^2 + \ltwo{\theta_0 - \theta^*}^2\right)\nonumber\\
	&\quad + \frac{256(4R^2_\theta\rho^2_{\max} + r^2_{\max})}{\min\{  \lambda_2\beta, \lambda_1\alpha \}}\left( \frac{32\alpha^2}{\lambda^2_2 \beta} + \frac{2\beta}{\lambda_2} + 2\beta^2 +  \frac{2\alpha}{\lambda_1} + 3\alpha^2 \right)\frac{1+(\kappa-1)\rho}{(1-\rho)M}.\label{eq: 32}
	\end{flalign}
\end{proof}

\section{Convergence Analysis of Two Time-scale Nonlinear TDC}\label{sc: proflemma2}
Before we present our technical proof of \Cref{thm2}, we first introduce some notations and definitions.
Recall that
\begin{flalign*}
\frac{1}{2}\nabla J(\theta_t) = -\mE[\delta(\theta_t)\phi_{\theta_t}(s)]-\gamma\mE[\phi_{\theta_t}(s^\prime)\phi_{\theta_t}(s)^\top ]w(\theta_t)+h(\theta_t,w(\theta_t)).
\end{flalign*}
For any $x_j=(s_j, a_j, s_{j+1})$, we define
\begin{flalign*}
g(\theta_t,w,x_j)=-\delta_j(\theta_t)\phi_{\theta_t}(s_j)-\gamma\phi_{\theta_t}(s_{j+1})\phi_{\theta_t}(s_j)^\top w + h(\theta_t,w,x_j),
\end{flalign*}
where $h(\theta_t,w,x_j)=(\delta_j(\theta_t)-\phi_{\theta_t}(s_j)^\top w)\nabla^2_{\theta_t} V_{\theta_t}(s_j)w$. We also define the mini-batch gradient estimator as $g(\theta_t,w, \mB_t)=\frac{1}{M}\sum_{j\in \mB_t} g(\theta_t, w,x_j)$, where $\mB_t=\{i_t, i_t+1, \cdots, i_t + M-1 \}$. 

For critic's update, we define $A_{\theta_t,x_j}=-\phi_{\theta_t}(s_j)\phi_{\theta_t}(s_j)^\top$, $b_{\theta_t,x_j}=\delta_j(\theta_t)\phi_{\theta_t}(s_j)$, $A_{\theta_t,\mB_k}=-\frac{1}{M}\sum_{j=i_k}^{i_k+M-1} \phi_{\theta_t}(s_j)\phi_{\theta_t}(s_j)^\top$, $b_{\theta_t,\mB_k}=\frac{1}{M}\sum_{j=i_k}^{i_k+M-1} \delta_j(\theta_t)\phi_{\theta_t}(s_j)$, $A_{\theta_t}=-\mE_{\mu_\pi}[\phi_{\theta_t}(s)\phi_{\theta_t}(s)^\top]$ and $b_{\theta_t}=\mE_{\mu_\pi}[\delta(\theta_t)\phi_{\theta_t}(s)]$. We also define $f_{\theta_t}(w^\prime_k,x_j)=A_{\theta_t,x_j}w^\prime_k + b_{\theta_t,x_j}$, $f_{\theta_t}(w^\prime_k,\mB_k)=A_{\theta_t,\mB_k}w^\prime_k+b_{\theta_t,\mB_k}$ and $f_{\theta_t}(w^\prime_k)=A_{\theta_t}w^\prime_k+b_{\theta_t}$. It can be checked easily that for all $\theta\in\mR^d$, we have $w(\theta)\leq R_w$, where $R_w=\frac{C_\phi(r_{\max} + 2C_v)}{\lambda_v}$.

\subsection{Preliminaries}
In this subsection, we provide some supporting lemmas, which are useful to the proof of \Cref{thm2}.
\begin{lemma}\label{lemma: boundedgradient}
	Suppose Assumptions \ref{ass1}-\ref{ass5} hold. For any $t\geq 0$ and $j$, we have $\ltwo{g(\theta_t, w(\theta_t), x_j)}\leq C_g$, where
	\begin{flalign*}
	C_g = [r_{\max} + (\gamma+1)C_v]C_\phi + \gamma C^2_\phi R_w + [r_{\max} + (\gamma+1)C_v + C_\phi R_w]D_vR_w.
	\end{flalign*}
\end{lemma}
\begin{proof}
	According to the definition of $g(\theta_t, w_t, x_j)$, we have
	\begin{flalign*}
	\ltwo{g(\theta_t, w(\theta_t), x_j)}&\leq \ltwo{\delta_j(\theta_t)\phi_{\theta_t}(s_j)} + \gamma\ltwo{\phi_{\theta_t}(s_{j+1})\phi_{\theta_t}(s_j)^\top w(\theta_t)} \\
	&\quad + \ltwo{(\delta_j(\theta_t)-\phi_{\theta_t}(s_j)^\top w(\theta_t)\nabla^2_{\theta_t} V_{\theta_t}(s_j)w(\theta_t)}\\
	&\leq \lone{\delta_j(\theta_t)}\ltwo{\phi_{\theta_t}(s_j)} + \gamma\ltwo{\phi_{\theta_t}(s_{j+1})} \ltwo{\phi_{\theta_t}(s_j)}\ltwo{w(\theta_t)} \\
	&\quad + \big(\lone{\delta_j(\theta_t)}+\lone{\phi_{\theta_t}(s_j)^\top w(\theta_t)}\big) \ltwo{\nabla^2_{\theta_t} V_{\theta_t}(s_j)}\ltwo{w(\theta_t)}\\
	&\overset{(i)}{\leq} [r_{\max} + (\gamma+1)C_v]C_\phi + \gamma C^2_\phi R_w + [r_{\max} + (\gamma+1)C_v + C_\phi R_w]D_vR_w.
	\end{flalign*}
	where $(i)$ follows from the fact that $w(\theta_t)\leq R_w$. 
\end{proof}

\begin{lemma}\label{lemma: fixpointlip}
	Suppose Assumptions \ref{ass1}-\ref{ass5} hold, for any $\theta, \theta^\prime\in \mR^d$, we have $\ltwo{w(\theta)-w(\theta^\prime)}\leq L_w\ltwo{\theta-\theta^\prime}$, where $L_w=\Big\{\frac{2C_\phi L_\phi}{\lambda^2_v}[r_{\max}+(1+\gamma)C_v] + \frac{1}{\lambda_v}[L_v C_\phi(1+\gamma) + L_\phi(r_{\max}+(1+\gamma)C_v)]\Big\}$.
\end{lemma}
\begin{proof}
	According to the definition of $w(\theta)$, we have
	\begin{flalign}
	\ltwo{w(\theta)-w(\theta^\prime)}&=\ltwo{A^{-1}_{\theta}b_{\theta} - A^{-1}_{\theta^\prime}b_{\theta^\prime} } = \ltwo{A^{-1}_{\theta}b_{\theta} - A^{-1}_{\theta^\prime}b_{\theta} + A^{-1}_{\theta^\prime}b_{\theta} - A^{-1}_{\theta^\prime}b_{\theta^\prime} }\nonumber\\
	&\leq \ltwo{A^{-1}_{\theta}b_{\theta} - A^{-1}_{\theta^\prime}b_{\theta}} + \ltwo{A^{-1}_{\theta^\prime}b_{\theta} - A^{-1}_{\theta^\prime}b_{\theta^\prime} }\nonumber\\
	&= \ltwo{A^{-1}_{\theta^\prime}A_{\theta^\prime}A^{-1}_{\theta}b_{\theta} - A^{-1}_{\theta^\prime}A_{\theta}A^{-1}_{\theta}b_{\theta}} + \ltwo{A^{-1}_{\theta^\prime}b_{\theta} - A^{-1}_{\theta^\prime}b_{\theta^\prime} }\nonumber\\
	&=\ltwo{A^{-1}_{\theta^\prime}(A_{\theta^\prime}-A_{\theta})A^{-1}_{\theta}b_{\theta}} + \ltwo{A^{-1}_{\theta^\prime}(b_{\theta} - b_{\theta^\prime}) }\nonumber\\
	&\leq \ltwo{A^{-1}_{\theta^\prime}}\ltwo{A_{\theta^\prime}-A_{\theta}}\ltwo{A^{-1}_{\theta}}\ltwo{b_{\theta}} + \ltwo{A^{-1}_{\theta^\prime}}\ltwo{b_{\theta} - b_{\theta^\prime} }\nonumber\\
	&\leq \frac{r_{\max}+(1+\gamma)C_v}{\lambda^2_v}\ltwo{A_{\theta^\prime}-A_{\theta}} + \frac{1}{\lambda_v}\ltwo{b_{\theta} - b_{\theta^\prime} }.\label{eq: 8}
	\end{flalign}
	Considering the term $\ltwo{A_{\theta^\prime}-A_{\theta}}$, by definition we can obtain
	\begin{flalign}
	\ltwo{A_{\theta^\prime}-A_{\theta}}&=\ltwo{\mE[\phi_\theta\phi_\theta^\top]-\mE[\phi_{\theta^\prime}\phi_{\theta^\prime}^\top]}=\ltwo{\mE[\phi_\theta\phi_\theta^\top] - \mE[\phi_{\theta^\prime}\phi_\theta^\top] + \mE[\phi_{\theta^\prime}\phi_\theta^\top] - \mE[\phi_{\theta^\prime}\phi_{\theta^\prime}^\top]}\nonumber\\
	&\leq \lF{\mE[\phi_\theta\phi_\theta^\top] - \mE[\phi_{\theta^\prime}\phi_\theta^\top]} + \lF{\mE[\phi_{\theta^\prime}\phi_\theta^\top] - \mE[\phi_{\theta^\prime}\phi_{\theta^\prime}^\top]}\nonumber\\
	&\leq 2\mE[\ltwo{\phi_\theta-\phi_{\theta^\prime}}\ltwo{\phi_\theta}]\leq 2C_\phi L_\phi\ltwo{\theta-\theta^\prime}.\label{eq: 9}
	\end{flalign}
	Considering the term $\ltwo{b_{\theta} - b_{\theta^\prime} }$, by definition we obtain
	\begin{flalign}
	\ltwo{b_{\theta} - b_{\theta^\prime} }&= \ltwo{\mE[\delta(\theta)\phi_\theta]-\mE[\delta(\theta^\prime)\phi_{\theta^\prime}]} = \ltwo{\mE[\delta(\theta)\phi_\theta]-\mE[\delta(\theta^\prime)\phi_\theta] + \mE[\delta(\theta^\prime)\phi_\theta] - \mE[\delta(\theta^\prime)\phi_{\theta^\prime}] }\nonumber\\
	&\leq \ltwo{\mE[\delta(\theta)\phi_\theta]-\mE[\delta(\theta^\prime)\phi_\theta]} + \ltwo{\mE[\delta(\theta^\prime)\phi_\theta] - \mE[\delta(\theta^\prime)\phi_{\theta^\prime}] }\nonumber\\
	&\leq \mE[\lone{\delta(\theta)-\delta(\theta^\prime)}\ltwo{\phi_\theta}] + \mE[\lone{\delta(\theta^\prime)}\ltwo{\phi_{\theta^\prime}-\phi_\theta}]\nonumber\\
	&=\mE[\lone{(\gamma V(s^\prime,\theta)-V(s,\theta)) - (\gamma V(s^\prime,\theta^\prime)-V(s,\theta^\prime)) }\ltwo{\phi_\theta}] + \mE[\lone{\delta(\theta^\prime)}\ltwo{\phi_{\theta^\prime}-\phi_\theta}]\nonumber\\
	&\leq [L_v C_\phi(1+\gamma) + L_\phi(r_{\max}+(1+\gamma)C_v)]\ltwo{\theta-\theta^\prime}.\label{eq: 10}
	\end{flalign}
	Substituting \cref{eq: 9} and \cref{eq: 10} into \cref{eq: 8} yields
	\begin{flalign*}
	&\ltwo{w(\theta)-w(\theta^\prime)}\\
	&\leq \Big\{\frac{2C_\phi L_\phi}{\lambda^2_v}[r_{\max}+(1+\gamma)C_v] + \frac{1}{\lambda_v}[L_v C_\phi(1+\gamma) + L_\phi(r_{\max}+(1+\gamma)C_v)]\Big\}\ltwo{\theta-\theta^\prime}.
	\end{flalign*}
\end{proof}

\begin{lemma}\label{lemma: trackingerror-iteration-sl}
	Suppose Assumptions \ref{ass1}-\ref{ass5} hold. Consider the iteration of $w_t$ in \Cref{al: nonlinearTDC}. Let the stepsize $\beta\leq \min\{\frac{\lambda_v}{8C^4_\phi}, \frac{8}{\lambda_v} \}$ and $\alpha\leq \frac{\lambda_v}{8\sqrt{2}L_w L_e}\beta$ and the batch size $M\geq (\frac{1}{\lambda_v}+2\beta)\frac{96C^4_\phi[1-(\kappa-1)\rho]}{\lambda_v(1-\rho)}$. For any $t>0$, we have
	\begin{flalign*}
	&\mE[\ltwo{w_t-w(\theta_t)}^2] \nonumber\\
	&\leq \left(1-\frac{\lambda_v}{8}\beta \right)\mE[\ltwo{w_{t-1}-w(\theta_{t-1})}^2] + \frac{2L^2_w\alpha^2}{\lambda_v \beta}\mE[\ltwo{\nabla J(\theta_{t-1})}^2] + \frac{D_1[1+(\kappa-1)\rho]}{M(1-\rho)},
	\end{flalign*}
	where $D_1=\frac{128L^2_w C^2_g\alpha^2}{\lambda_v \beta} +  4C^2_f\left(\frac{\beta}{\lambda_v}+2\beta^2\right)$.
\end{lemma}
\begin{proof}
	We proceed as follows:
	\begin{flalign}
	&\ltwo{w_{t}-w(\theta_{t-1})}^2\nonumber\\
	&=\ltwo{w_{t-1}+\beta f_{\theta_{t-1}}(w_{t-1},\mB_t)-w(\theta_{t-1})}^2\nonumber\\
	&=\ltwo{w_{t-1}-w(\theta_{t-1})}^2 + 2\beta\langle w_{t-1}-w(\theta_{t-1}), f_{\theta_{t-1}}(w_{t-1},\mB_t)  \rangle + \beta^2\ltwo{f_{\theta_{t-1}}(w_{t-1},\mB_t)}^2\nonumber\\
	&=\ltwo{w_{t-1}-w(\theta_{t-1})}^2 + 2\beta\langle w_{t-1}-w(\theta_{t-1}), f_{\theta_{t-1}}(w_{t-1})  \rangle \nonumber\\
	&\quad + 2\beta\langle w_{t-1}-w(\theta_{t-1}), f_{\theta_{t-1}}(w_{t-1},\mB_t)-f_{\theta_{t-1}}(w_{t-1})  \rangle \nonumber\\
	&\quad + \beta^2\ltwo{f_{\theta_{t-1}}(w_{t-1},\mB_t)-f_{\theta_{t-1}}(w_{t-1})+f_{\theta_{t-1}}(w_{t-1})}^2\nonumber\\
	&\overset{(i)}{\leq} (1-2\lambda_v\beta)\ltwo{w_{t-1}-w(\theta_{t-1})}^2 + 2\beta\langle w_{t-1}-w(\theta_{t-1}), f_{\theta_{t-1}}(w_{t-1},\mB_t)-f_{\theta_{t-1}}(w_{t-1})  \rangle \nonumber\\
	&\quad + \beta^2\ltwo{f_{\theta_{t-1}}(w_{t-1},\mB_t)-f_{\theta_{t-1}}(w_{t-1})+f_{\theta_{t-1}}(w_{t-1})}^2\nonumber\\
	&\overset{(ii)}{\leq} (1-2\lambda_v\beta)\ltwo{w_{t-1}-w(\theta_{t-1})}^2 + \lambda_v\beta\ltwo{w_{t-1}-w(\theta_{t-1})}^2 + \frac{\beta}{\lambda_v}\ltwo{f_{\theta_{t-1}}(w_{t-1},\mB_t)-f_{\theta_{t-1}}(w_{t-1}) }^2\nonumber\\
	&\quad +2\beta^2\ltwo{f_{\theta_{t-1}}(w_{t-1},\mB_t)-f_{\theta_{t-1}}(w_{t-1})} + 2\beta^2\ltwo{f_{\theta_{t-1}}(w_{t-1})}^2\nonumber\\
	&\overset{(iii)}{=}(1-\lambda_v\beta + 2C^4_\phi\beta^2)\ltwo{w_{t-1}-w(\theta_{t-1})}^2 + \left(\frac{\beta}{\lambda_v}+2\beta^2\right)\ltwo{f_{\theta_{t-1}}(w_{t-1},\mB_t)-f_{\theta_{t-1}}(w_{t-1})}^2,\label{eq: 5}
	\end{flalign}
	where $(i)$ follows from the fact that
	\begin{flalign*}
	\langle w_{t-1}-w(\theta_{t-1}), f_{\theta_{t-1}}(w_{t-1})  \rangle &= \langle w_{t-1}-w(\theta_{t-1}), A_{\theta_{t-1}}(w_{t-1}-w(\theta_{t-1}))  \rangle \nonumber\\
	&\leq -\lambda_v\ltwo{w_{t-1}-w(\theta_{t-1})}^2,
	\end{flalign*}
	$(ii)$ follows from the fact that $\langle a,b \rangle\leq \frac{\lambda_v}{2}a^2 + \frac{1}{2\lambda_v}b^2$, and $(iii)$ follows from the fact that $\ltwo{f_{\theta_{t-1}}(w_{t-1})}^2=\ltwo{A_{\theta_{t-1}}(w_{t-1}-w(\theta_{t-1}))}^2\leq C^4_\phi\ltwo{w_{t-1}-w(\theta_{t-1})}^2$. Taking expectation on both side of \cref{eq: 5} yields
	\begin{flalign}
	&\mE[\ltwo{w_{t}-w(\theta_{t-1})}^2] \nonumber\\
	&\leq (1-\lambda_v\beta + 2C^4_\phi\beta^2)\mE[\ltwo{w_{t-1}-w(\theta_{t-1})}^2] + \left(\frac{\beta}{\lambda_v}+2\beta^2\right)\mE\left[\ltwo{f_{\theta_{t-1}}(w_{t-1},\mB_t)-f_{\theta_{t-1}}(w_{t-1})}^2\right].\label{eq: 11}
	\end{flalign}
	Next we bound the term $\mE\left[\ltwo{f_{\theta_{t-1}}(w_{t-1},\mB_t)-f_{\theta_{t-1}}(w_{t-1})}^2\right]$ in \cref{eq: 11} as follows:
	\begin{flalign}
	&\mE\left[\ltwo{f_{\theta_{t-1}}(w_{t-1},\mB_t)-f_{\theta_{t-1}}(w_{t-1})}^2\right]\nonumber\\
	&=\mE\left[\ltwo{(A_{\theta_{t-1},\mB_{t-1}} - A_{\theta_{t-1}})w_{t-1} + b_{\theta_{t-1},\mB_{t-1} } - b_{\theta_{t-1} } }^2\right]\nonumber\\
	&=\mE\left[\ltwo{(A_{\theta_{t-1},\mB_{t-1}} - A_{\theta_{t-1}})(w_{t-1} - w(\theta_{t-1}) ) + (A_{\theta_{t-1},\mB_{t-1}} - A_{\theta_{t-1}}) w(\theta_{t-1}) + b_{\theta_{t-1},\mB_{t-1} } - b_{\theta_{t-1} } }^2\right]\nonumber\\
	&\leq 3\mE\left[\ltwo{(A_{\theta_{t-1},\mB_{t-1}} - A_{\theta_{t-1}})(w_{t-1} - w(\theta_{t-1}) )}^2\right] + 3\mE\left[\ltwo{(A_{\theta_{t-1},\mB_{t-1}} - A_{\theta_{t-1}}) w(\theta_{t-1})}^2\right] \nonumber\\
	&\quad + 3 \mE\left[\ltwo{b_{\theta_{t-1},\mB_{t-1} } - b_{\theta_{t-1} } }^2\right].\label{eq: 13}
	\end{flalign}
	From \Cref{ass2}, we have $\lF{A_{\theta_t,x_j}}\leq C^2_\phi$ and $\ltwo{b_{\theta_t,x_j}}\leq C_\phi(r_{\max} + 2C_v)$. Following from \Cref{lemma1}, we can obtain the following two upper bounds:
	\begin{flalign}
	\mE\left[\ltwo{(A_{\theta_{t-1},\mB_{t-1}} - A_{\theta_{t-1}}) }^2\right] \leq \frac{8C^4_\phi [1-(\kappa-1)\rho]}{(1-\rho)M},\label{eq: 14}
	\end{flalign}
	and
	\begin{flalign}
	\mE\left[\ltwo{b_{\theta_{t-1},\mB_{t-1} } - b_{\theta_{t-1} } }^2\right] \leq \frac{8C^2_\phi (r_{\max} + 2C_v)^2[1-(\kappa-1)\rho]}{(1-\rho)M}. \label{eq : 15}
	\end{flalign}
	Substituting \cref{eq: 14} and \cref{eq : 15} into \cref{eq: 13} yields
	\begin{flalign}
	&\mE\left[\ltwo{f_{\theta_{t-1}}(w_{t-1},\mB_t)-f_{\theta_{t-1}}(w_{t-1})}^2\right]\nonumber\\
	&=\frac{24C^4_\phi [1-(\kappa-1)\rho]}{(1-\rho)M}\mE[\ltwo{w_{t-1} - w(\theta_{t-1})}^2] + \frac{24 [C^2_\phi (r_{\max} + 2C_v)^2 + C^4_\phi R_w] [1-(\kappa-1)\rho]}{(1-\rho)M}.\label{eq: 18}
	\end{flalign}
	Substituting \cref{eq: 18} into \cref{eq: 5} yields
	\begin{flalign}
	&\mE[\ltwo{w_{t}-w(\theta_{t-1})}^2]\nonumber\\
	&\leq \left(1-\lambda_v\beta + 2C^4_\phi\beta^2 + \left(\frac{\beta}{\lambda_v}+2\beta^2\right) \frac{24C^4_\phi [1-(\kappa-1)\rho]}{(1-\rho)M} \right)\mE[\ltwo{w_{t-1}-w(\theta_{t-1})}^2] \nonumber\\
	&\quad + \left(\frac{\beta}{\lambda_v}+2\beta^2\right) \frac{24 [C^2_\phi (r_{\max} + 2C_v)^2 + C^4_\phi R_w] [1-(\kappa-1)\rho]}{(1-\rho)M}\nonumber\\
	&\overset{(i)}{\leq} \Big(1-\frac{\lambda_v}{2}\beta\Big)\mE[\ltwo{w_{t-1}-w(\theta_{t-1})}^2] + \left(\frac{\beta}{\lambda_v}+2\beta^2\right) \frac{ 4C_f [1-(\kappa-1)\rho]}{(1-\rho)M}, \label{eq: 6}
	\end{flalign}
	where $(i)$ follows from the fact that $\beta\leq \frac{\lambda_v}{8C^4_\phi}$ and $M\geq (\frac{1}{\lambda_v}+2\beta)\frac{96C^4_\phi[1-(\kappa-1)\rho]}{\lambda_v(1-\rho)}$, and here we define $C_f=6 [C^2_\phi (r_{\max} + 2C_v)^2 + C^4_\phi R_w]$.
	By Young's inequality, we have
	\begin{flalign}
	&\mE[\ltwo{w_t-w(\theta_t)}^2] \nonumber\\
	&\leq \left(1 + \frac{1}{2(2/(\lambda_v\beta)-1)}\right)\mE[\ltwo{w_t-w(\theta_{t-1})}^2] + (1 + 2(2/(\lambda_v\beta)-1))\mE[\ltwo{w(\theta_{t-1})-w(\theta_t)}^2]\nonumber\\
	&\overset{(i)}{\leq} \left( \frac{4/(\lambda_v\beta)-1}{4/(\lambda_v\beta)-2} \right)\mE[\ltwo{w_{t-1}-w(\theta_{t-1})}^2] + \frac{4}{\lambda_v \beta}\mE[\ltwo{w(\theta_{t-1})-w(\theta_t)}^2] \nonumber\\
	&\quad + \left( \frac{4/(\lambda_v\beta)-1}{4/(\lambda_v\beta)-2} \right)\left(\frac{\beta}{\lambda_v}+2\beta^2\right)\frac{4C^2_f[1+(\kappa-1)\rho]}{M(1-\rho)} \nonumber\\
	&\overset{(ii)}{\leq} \Big(1-\frac{\lambda_v}{4}\beta\Big) \mE[\ltwo{w_{t-1}-w(\theta_{t-1})}^2] + \frac{4L^2_w}{\lambda_v \beta}\mE[\ltwo{\theta_{t-1}-\theta_t}^2] \nonumber\\
	&\quad + \left(\frac{\beta}{\lambda_v}+2\beta^2\right) \frac{4C^2_f[1+(\kappa-1)\rho]}{M(1-\rho)} \nonumber\\
	&\leq \Big(1-\frac{\lambda_v}{4}\beta\Big) \mE[\ltwo{w_{t-1}-w(\theta_{t-1})}^2] + \frac{2L^2_w\alpha^2}{\lambda_v \beta}\mE[\ltwo{\nabla J(\theta_{t-1})}^2] \nonumber\\
	&\quad + \frac{8L^2_w\alpha^2}{\lambda_v \beta}\mE\left[\ltwo{g(\theta_{t-1}, w_{t-1}, \mB_{t-1})-\frac{1}{2}\nabla J(\theta_{t-1})}^2\right] + \left(\frac{\beta}{\lambda_v}+2\beta^2\right)\frac{4C^2_f[1+(\kappa-1)\rho]}{M(1-\rho)},\label{eq: 19}
	\end{flalign}
	where $(i)$ follows from \cref{eq: 14} and $(ii)$ follows from \Cref{lemma: fixpointlip}. We next bound the third term on the right hand side of \cref{eq: 19} as follows:
	\begin{flalign}
	&\mE\left[\ltwo{g(\theta_{t-1}, w_{t-1}, \mB_{t-1})-\frac{1}{2}\nabla J(\theta_{t-1})}^2\right]\nonumber\\
	&\leq 2\mE\left[\ltwo{g(\theta_{t-1}, w_{t-1}, \mB_{t-1})-g(\theta_{t-1}, w(\theta_{t-1}), \mB_{t-1}) }^2\right] + 2\mE\left[\ltwo{g(\theta_{t-1}, w(\theta_{t-1}), \mB_{t-1})-\frac{1}{2}\nabla J(\theta_{t-1})}^2\right] \nonumber\\
	&\overset{(i)}{\leq} 2L^2_e \mE\left[\ltwo{w_{t-1} - w(\theta_{t-1}) }^2 \right] + \frac{16C^2_g[1+(\kappa-1)\rho]}{M(1-\rho)}.\label{eq: 20}
	\end{flalign}
	Substituting \cref{eq: 20} into \cref{eq: 19} yields
	\begin{flalign}
	&\mE[\ltwo{w_t-w(\theta_t)}^2] \nonumber\\
	&\leq \left(1-\frac{\lambda_v}{4}\beta + \frac{16L^2_w L^2_e\alpha^2}{\lambda_v \beta} \right) \mE[\ltwo{w_{t-1}-w(\theta_{t-1})}^2] + \frac{2L^2_w\alpha^2}{\lambda_v \beta}\mE[\ltwo{\nabla J(\theta_{t-1})}^2] \nonumber\\
	&\quad + \left[\frac{128L^2_w C^2_g\alpha^2}{\lambda_v \beta} +  4C^2_f\left(\frac{\beta}{\lambda_v}+2\beta^2\right)\right] \frac{[1+(\kappa-1)\rho]}{M(1-\rho)}\nonumber\\
	&\overset{(i)}{\leq} \left(1-\frac{\lambda_v}{8}\beta \right)\mE[\ltwo{w_{t-1}-w(\theta_{t-1})}^2] + \frac{2L^2_w\alpha^2}{\lambda_v \beta}\mE[\ltwo{\nabla J(\theta_{t-1})}^2] + \frac{D_1[1+(\kappa-1)\rho]}{M(1-\rho)}, \label{eq: 21}
	\end{flalign}
	where $(i)$ follows from the fact that $\alpha\leq \frac{\lambda_v}{8\sqrt{2}L_w L_e}\beta$ and we define $D_1=\frac{128L^2_w C^2_g\alpha^2}{\lambda_v \beta} +  4C^2_f(\frac{\beta}{\lambda_v}+2\beta^2)$.
\end{proof}

\subsection{Proof of \Cref{thm2}}
Since $J(\theta)$ is $L_J$-gradient Lipschitz, we have
\begin{flalign}
&\mE[J(\theta_{t+1})] \nonumber\\
&\leq \mE[J(\theta_t)] + \mE[\langle \nabla J(\theta_t), \theta_{t+1}-\theta_t \rangle] + \frac{L_J}{2}\mE[\ltwo{\theta_{t+1}-\theta_t}^2]\nonumber\\
&=\mE[J(\theta_t)] - \frac{\alpha}{2} \mE[\ltwo{\nabla J(\theta_t)}^2] - \alpha \mE[ \langle \nabla J(\theta_t), g(\theta_t,w_t,\mB_t) - \frac{1}{2}\nabla J(\theta_t) \rangle ]+ \frac{L_J \alpha^2}{2}\mE[\ltwo{g(\theta_t,w_t,\mB_t)}^2]\nonumber\\
&\leq \mE[J(\theta_t)]-\Big(\frac{\alpha}{4} - \frac{L_J\alpha^2}{8} \Big)\mE[\ltwo{\nabla J(\theta_t)}^2]+(\alpha + L_J\alpha^2 )\mE\left[\ltwo{g(\theta_t,w_t,\mB_t) - \frac{1}{2}\nabla J(\theta_t) }^2\right]\nonumber\\
&\leq \mE[J(\theta_t)]-\Big(\frac{\alpha}{4} - \frac{L_J\alpha^2}{8} \Big)\mE[\ltwo{\nabla J(\theta_t)}^2] + 2( \alpha + L_J\alpha^2 )\mE[\ltwo{g(\theta_t,w_t,\mB_t) - g(\theta_t,w(\theta_t),\mB_t) }^2] \nonumber\\
&\quad  + 2( \alpha + L_J\alpha^2 )\mE\left[\ltwo{ g(\theta_t,w(\theta_t),\mB_t) - \frac{1}{2}\nabla J(\theta_t) }^2\right] \nonumber\\
&\overset{(i)}{\leq} \mE[J(\theta_t)]-\Big(\frac{\alpha}{4} - \frac{L_J\alpha^2}{8} \Big)\mE[\ltwo{\nabla J(\theta_t)}^2 ] + 2( \alpha + L_J\alpha^2 )L^2_e \mE[\ltwo{w_t - w(\theta_t)}^2] \nonumber \\
&\quad + 2(\alpha + L_J\alpha^2 )\frac{4C^2_g[1+(\kappa-1)\rho]}{M(1-\rho)},\label{eq: 24}
\end{flalign}
where $(i)$ follows from \Cref{ass5} and \Cref{lemma1}. Rearranging the above inequality and summing from $t=0$ to $T-1$ yield
\begin{flalign}
\Big(\frac{\alpha}{4} - \frac{L_J\alpha^2}{8} \Big)\sum_{t=0}^{T-1}\mE[\ltwo{\nabla J(\theta_t)}^2]&\leq J(\theta_0)-J(\theta_T) + 2( \alpha + L_J\alpha^2 )T\frac{4C^2_g[1+(\kappa-1)\rho]}{M(1-\rho)}\nonumber\\
&\quad + 2( \alpha + L_J\alpha^2 )L^2_e \sum_{t=0}^{T-1}\mE\ltwo{w_t - w(\theta_t)}^2.\label{eq: 15}
\end{flalign}
Now we upper bound the term $\sum_{t=0}^{T-1}\mE\ltwo{w_t - w(\theta_t)}^2$. Applying the inequality in \Cref{lemma: trackingerror-iteration-sl} recursively yields
\begin{flalign*}
\mE[\ltwo{w_t-w(\theta_t)}^2] &\leq \Big(1-\frac{\lambda_v}{8}\beta\Big)^t \ltwo{w_0-w(\theta_0)}^2 + \frac{2L^2_w\alpha^2}{\lambda_v \beta}\sum_{i=0}^{t-1} \Big(1-\frac{\lambda_v}{8}\beta\Big)^{t-1-i} \mE[\ltwo{\nabla J(\theta_{t-1})}^2] \\
&\quad + \frac{D_1[1+(\kappa-1)\rho]}{M(1-\rho)} \sum_{i=0}^{t-1}\Big(1-\frac{\lambda_v}{8}\beta\Big)^{t-1-i},
\end{flalign*}
which implies
\begin{flalign}
\sum_{t=0}^{T-1}\mE\ltwo{w_t - w(\theta_t)}^2&\leq \ltwo{w_0-w(\theta_0)}^2\sum_{t=0}^{T-1}\Big(1-\frac{\lambda_v}{8}\beta\Big)^t + \frac{2L^2_w\alpha^2}{\lambda_v \beta}\sum_{t=0}^{T-1}\sum_{i=0}^{t-1} \Big(1-\frac{\lambda_v}{8}\beta\Big)^{t-1-i} \mE[\ltwo{\nabla J(\theta_{t-1})}^2]\nonumber\\
&\quad + \frac{D_1[1+(\kappa-1)\rho]}{M(1-\rho)} \sum_{t=0}^{T-1} \sum_{i=0}^{t-1}\Big(1-\frac{\lambda_v}{8}\beta\Big)^{t-1-i}\nonumber\\
&\leq \frac{8\ltwo{w_0-w(\theta_0)}^2}{\lambda_v\beta}+ \frac{16L^2_w\alpha^2}{\lambda^2_v \beta^2}\sum_{t=0}^{T-1}\mE[\ltwo{\nabla J(\theta_{t-1})}^2]+ \frac{8D_1T}{\lambda_v\beta}\frac{1+(\kappa-1)\rho}{M(1-\rho)}. \label{eq: 16}
\end{flalign}
Substituting \cref{eq: 16} into \cref{eq: 15} yields
\begin{flalign}
&\Big(\frac{\alpha}{4} - \frac{L_J\alpha^2}{8} - \frac{32L^2_w L^2_e \alpha^3(1+L_J\alpha)}{\lambda^2_v \beta^2} \Big)\sum_{t=0}^{T-1}\mE[\ltwo{\nabla J(\theta_t)}^2]\nonumber\\
&\leq J(\theta_0)-\mE[J(\theta_T)] + \frac{16( \alpha + L_J\alpha^2 )L^2_e}{\lambda_v\beta}\ltwo{w_0-w(\theta_0)}^2 + 2( \alpha + L_J\alpha^2 )T\frac{4C^2_g[1+(\kappa-1)\rho]}{M(1-\rho)}\nonumber\\
&\quad + \frac{16D_1( \alpha + L_J\alpha^2 )L^2_eT}{\lambda_v\beta}\frac{1+(\kappa-1)\rho}{M(1-\rho)}.\label{eq: 17}
\end{flalign}
Dividing both sides of \cref{eq: 17} by $T$ and using the fact that $\frac{\alpha}{4} - \frac{L_J\alpha^2}{8} - \frac{32L^2_w L^2_e \alpha^3(1+L_J\alpha)}{\lambda^2_v \beta^2}\geq \frac{\alpha}{8}$, we have
\begin{flalign}
&\frac{1}{T}\sum_{t=0}^{T-1}\mE[\ltwo{\nabla J(\theta_t)}^2]\nonumber\\
&\leq \frac{8(J(\theta_0)-\mE[J(\theta_T)])}{\alpha T} + \frac{64( 1 + L_J\alpha )L^2_e}{\lambda_v\beta}\frac{\ltwo{w_0-w(\theta_0)}^2}{T} + 64( 1 + L_J\alpha )\Big(C^2_g + \frac{2D_1L^2_e}{\lambda_v\beta} \Big)\frac{1+(\kappa-1)\rho}{M(1-\rho)}.\label{eq: 33}
\end{flalign}

\section{Convergence Analysis of Two Time-scale Greedy-GQ}\label{sc: appGQ}
We make the following definitions. For a given $\theta$, we define matrices $A_\theta=\mE_{\mu_{\pi_b}}[(\gamma\mE_{\pi_\theta}[\phi(s^\prime)|s]-\phi(s))\phi(s)^\top]$, $B_\theta = \mE_{\mu_{\pi_b}}[\mE_{\pi_\theta}[\phi(s^\prime)|s]\phi(s)^\top]$, $C=-\mE_{\mu_{\pi_b}}[\phi(s)\phi(s)^\top]$ and vectors $b_\theta=\mE_{\mu_{\pi_b}}[\mE_{\pi_\theta}[r(s^\prime,s)|s]\phi(s)]$, $w^*(\theta)=C^{-1}(A_\theta \theta + b_\theta)$, $\theta^*=-A^{-1}_\theta b_\theta$. 
We also define the stochastic matrices $A_t = \frac{1}{\lone{\mB_t}}\sum_{j\in \mB_t}\gamma\rho_{\theta_t}(s_j, a_j)\phi(s_{j+1})\phi(s_j)^\top - \phi(s_j)\phi(s_j)^\top $, $B_t= \frac{1}{\lone{\mB_t}}\sum_{j\in \mB_t}\rho_{\theta_t}(s_j, a_j)\phi(s_{j+1})\phi(s_j)^\top$, $C_t=\frac{1}{\lone{\mB_t}}\sum_{j\in \mB_t}\phi(s_j)\phi(s_j)^\top$ and stochastic vector $b_t =\frac{1}{\lone{\mB_t}}\sum_{j\in \mB_t} \rho_{\theta_t}(s_j,a_j)r(s_{j+1},s_j)\phi(s_j)$. 

We also define the full (semi)-gradient as follows:
\begin{flalign}
-\frac{1}{2}\nabla J(\theta)=g(\theta)&= (A_\theta-B_\theta C^{-1} A_\theta)\theta + (b_\theta - B_\theta C^{-1} b_\theta), \label{gq: g}\\
f(w)&= C(w-w^*(\theta)), \label{gq: f}
\end{flalign}
and stochastic (semi)-gradient at step $t$ as follows:
\begin{flalign}
g_t(\theta_t)&= (A_t-B_tC^{-1}A_{\theta_t})\theta_t + (b_t - B_tC^{-1}b_{\theta_t}), \label{gq: gt}\\
f_t(w_t)&= C_t(w_t-w^*(\theta_t)), \label{gq: ft}\\
h_t(\theta_t)&=(A_t-C_tC^{-1}A_{\theta_t})\theta_t + (b_t - C_tC^{-1}b_{\theta_t}). \label{gq: ht}
\end{flalign}

We first consider the induction relationship for the fast time-scale variable $w_t$. Following similar steps from \cref{eq: 2} to \cref{eq: 7}, letting $M\geq 128\left( \rho^2_{\max} + \frac{1}{\lambda^2_2} \right)\frac{1+(\kappa-1)\rho}{1-\rho}\max\{1, \frac{\lambda^2_2\beta}{4\alpha^2}(\frac{2\beta}{\lambda_2} + 2\beta^2) \}$ and $\beta\leq \frac{\lambda_2}{4}$, we obtain
\begin{flalign}
&\mE[\ltwo{w_{t+1}-w^*(\theta_{t+1})}^2] \nonumber\\
&\leq \left( 1- \frac{\lambda_2\beta}{4} + \frac{16\rho^2_{\max}\alpha^2}{\lambda^2_2 \beta} \right)\mE[\ltwo{w_t - w^*(\theta_t)}^2] +  \frac{100\alpha^2}{\lambda^2_2 \beta}  \mE[\ltwo{\theta_t-\theta^*}^2] \nonumber\\
&\quad + 32(4R^2_\theta\rho^2_{\max} + r^2_{\max})\left(  \frac{32\alpha^2}{\lambda^2_2 \beta} + \frac{2\beta}{\lambda_2} + 2\beta^2 \right)\frac{1+(\kappa-1)\rho}{(1-\rho)M}\nonumber\\
&\overset{(i)}{\leq} \left( 1- \frac{\lambda_2\beta}{8}  \right) \mE[\ltwo{w_t - w^*(\theta_t)}^2] +  \frac{100\lambda^2_1\alpha^2}{\lambda^2_2 \beta}  \mE[\ltwo{\nabla J(\theta_t)}^2] \nonumber\\
&\quad + 32(4R^2_\theta\rho^2_{\max} + r^2_{\max})\left(  \frac{32\alpha^2}{\lambda_2^2 \beta} + \frac{2\beta}{\lambda_2} + 2\beta^2 \right)\frac{1+(\kappa-1)\rho}{(1-\rho)M},\label{eq: 22}
\end{flalign}
where $(i)$ follows from the fact that $\alpha\leq \frac{\lambda_2\sqrt{\lambda_2}}{8\sqrt{2}\rho_{\max}}\beta$ and $\ltwo{\theta_t-\theta^*}\leq \lambda_1 \ltwo{\nabla J(\theta_t)}$ according to the definition of $\nabla J(\theta)$ in \ref{gq: g}.
We next consider the induction relationship for the slow time-scale variable $\theta_t$. Since $J(\theta)$ is $L_J$-gradient Lipschitz, we have
\begin{flalign}
&\mE[J(\theta_{t+1})] \nonumber\\
&\leq \mE[J(\theta_t)] + \mE[\langle \nabla J(\theta_t), \theta_{t+1}-\theta_t \rangle] + \frac{L_J}{2}\mE[\ltwo{\theta_{t+1}-\theta_t}^2]\nonumber\\
&=\mE[J(\theta_t)] - \frac{\alpha}{2} \mE[\ltwo{\nabla J(\theta_t)}^2] - \alpha \mE[ \langle \nabla J(\theta_t), -g_t(\theta_t) - \frac{1}{2}\nabla J(\theta_t) \rangle ]  \nonumber\\
&\quad + \alpha \mE[ \langle \nabla J(\theta_t), B_t(w_t - w^*(\theta_t)) \rangle ] + \frac{L_J \alpha^2}{2}\mE[\ltwo{g_t(\theta_t) + B_t(w_t - w^*(\theta_t))}^2]\nonumber\\
&\overset{(i)}{\leq}\mE[J(\theta_t)] - \frac{\alpha}{4} \mE[\ltwo{\nabla J(\theta_t)}^2] + 2\alpha \mE\left[  \ltwo{-g_t(\theta_t) - \frac{1}{2}\nabla J(\theta_t)}^2  \right]  \nonumber\\
&\quad + 2\alpha \mE[ \ltwo{B_t(w_t - w^*(\theta_t))}^2  ] + L_J \alpha^2\mE[\ltwo{g_t(\theta_t) }^2] +  L_J \alpha^2\mE[\ltwo{B_t(w_t - w^*(\theta_t))}^2]\nonumber\\
&\overset{(ii)}{\leq} \mE[J(\theta_t)] - \left(\frac{\alpha}{4} - \frac{L_J \alpha^2}{2} \right) \mE[\ltwo{\nabla J(\theta_t)}^2] + 2(\alpha + L_J\alpha^2) \mE\left[  \ltwo{-g_t(\theta_t) - \frac{1}{2}\nabla J(\theta_t)}^2  \right]  \nonumber\\
&\quad + (2\alpha +  L_J \alpha^2)\rho^2_{\max}\mE[\ltwo{w_t - w^*(\theta_t)}^2],\label{eq: 25}
\end{flalign}
where $(i)$ follows from Young's inequality and $(ii)$ follows from the fact that $\ltwo{g_t(\theta_t) }^2\leq \frac{1}{2}\ltwo{\nabla J(\theta_t) }^2 + 2\ltwo{-g_t(\theta_t) - \frac{1}{2}\nabla J(\theta_t) }^2$ and $\ltwo{B_t}\leq \rho_{\max}$. Then, we upper bound the term $ \mE[  \ltwo{-g_t(\theta_t) - \frac{1}{2}\nabla J(\theta_t)}^2 ]$ as follows:
\begin{flalign}
&\mE\left[  \ltwo{-g_t(\theta_t) - \frac{1}{2}\nabla J(\theta_t)}^2  \right]\nonumber\\
&= \mE\left[\ltwo{\left[(A_t-A_{\theta_t})-(B_t-B_{\theta_t})C^{-1}_{\theta_t}A_{\theta_t}\right]\theta_t + \left[(b_t-b_{\theta_t}) - (B_t - B_{\theta_t})C^{-1}_{\theta_t}b_{\theta_t}\right] }^2\right]\nonumber\\
&\leq 4\mE\left[ \ltwo{(A_t-A_{\theta_t})\theta_t}^2 \right] + 4\mE\left[ \ltwo{(B_t-B_{\theta_t})C^{-1}_{\theta_t}A_{\theta_t}\theta_t}^2 \right] + 4\mE\left[ \ltwo{b_t-b_{\theta_t}}^2 \right] \nonumber\\
&\quad + 4\mE\left[ \ltwo{(B_t - B_{\theta_t})C^{-1}_{\theta_t}b_{\theta_t}}^2 \right]\nonumber\\
&\leq 4\mE\left[ \ltwo{A_t-A_{\theta_t}}^2\ltwo{\theta_t}^2 \right] + 4\mE\left[ \ltwo{B_t-B_{\theta_t}}^2\ltwo{C^{-1}_{\theta_t}}^2\ltwo{A_{\theta_t}}^2\ltwo{\theta_t}^2 \right] + 4\mE\left[ \ltwo{b_t-b_{\theta_t}}^2 \right] \nonumber\\
&\quad + 4\mE\left[ \ltwo{B_t - B_{\theta_t}}^2\ltwo{C^{-1}_{\theta_t}}^2\ltwo{b_{\theta_t}}^2 \right]\nonumber\\
&= 4\mE\left[ \mE[\ltwo{A_t-A_{\theta_t}}^2|\mf_t]\ltwo{\theta_t}^2 \right] + 4\mE\left[ \mE[\ltwo{B_t-B_{\theta_t}}^2|\mf_t]\ltwo{C^{-1}_{\theta_t}}^2\ltwo{A_{\theta_t}}^2\ltwo{\theta_t}^2 \right] + 4\mE\left[ \ltwo{b_t-b_{\theta_t}}^2 \right] \nonumber\\
&\quad + 4\mE\left[ \mE[\ltwo{B_t - B_{\theta_t}}^2|\mf_t]\ltwo{C^{-1}_{\theta_t}}^2\ltwo{b_{\theta_t}}^2 \right]\nonumber\\
&\leq \frac{32(\rho_{\max}+1)^2[1+(\kappa-1)\rho]}{(1-\rho)M}\mE\left[\ltwo{\theta_t}^2\right] + \frac{32(\rho_{\max}+1)^2\rho_{\max}^2[1+(\kappa-1)\rho]}{(1-\rho)\lambda^2_2M}\mE\left[\ltwo{\theta_t}^2\right] \nonumber\\
&\quad + \frac{32r^2_{\max}\rho^2_{\max}[1+(\kappa-1)\rho]}{(1-\rho)M} + \frac{32\rho_{\max}^2[1+(\kappa-1)\rho]}{(1-\rho)\lambda^2_2M}\nonumber\\
&\leq \frac{32(\rho_{\max}+1)^4[1+(\kappa-1)\rho]}{(1-\rho)M}\mE\left[\ltwo{\theta_t}^2\right] + \frac{32(r^2_{\max}+1)\rho^2_{\max}[1+(\kappa-1)\rho]}{(1-\rho)M} \nonumber\\
&\leq \frac{64(\rho_{\max}+1)^4[1+(\kappa-1)\rho]}{(1-\rho)M}\mE\left[\ltwo{\theta^*_t}^2\right] + \frac{64(\rho_{\max}+1)^4[1+(\kappa-1)\rho]}{(1-\rho)M}\mE\left[\ltwo{\theta_t - \theta^*_t}^2\right] \nonumber\\
&\quad + \frac{32(r^2_{\max}+1)\rho^2_{\max}[1+(\kappa-1)\rho]}{(1-\rho)M}\nonumber\\
&\leq \frac{C_1 [1+(\kappa-1)\rho]}{(1-\rho)M} + \frac{64\lambda^2_1(\rho_{\max}+1)^4[1+(\kappa-1)\rho]}{(1-\rho)M}\mE\left[\ltwo{\nabla J(\theta_t)}^2\right], \label{eq: 26}
\end{flalign}
where $C_2 = 32 [2(\rho_{\max}+1)^4 R^2_\theta + (r^2_{\max}+1)\rho^2_{\max}]$. Substituting \cref{eq: 26} into \cref{eq: 25}, rearranging the terms and summing from $t=0$ to $T-1$ yield
\begin{flalign}
&\Big(\frac{\alpha}{4} - \frac{L_J\alpha^2}{2} \Big)\sum_{t=0}^{T-1}\mE[\ltwo{\nabla J(\theta_t)}^2] \nonumber\\
&\leq J(\theta_0)-\mE[J(\theta_T)] + 2( \alpha + L_J\alpha^2 )T\frac{C_2[1+(\kappa-1)\rho]}{M(1-\rho)} + ( 2\alpha + L_J\alpha^2 ) \rho^2_{\max} \sum_{t=0}^{T-1}\mE\ltwo{w_t - w(\theta_t)}^2 \nonumber\\
&\quad + 2( \alpha + L_J\alpha^2 )\frac{64\lambda^2_1(\rho_{\max}+1)^4[1+(\kappa-1)\rho]}{(1-\rho)M}\sum_{t=0}^{T-1}\mE[\ltwo{\nabla J(\theta_t)}^2].\label{eq: 27}
\end{flalign}
Then, we bound the term $\sum_{t=0}^{T-1}\mE\ltwo{w_t - w(\theta_t)}^2$. Applying \cref{eq: 22} iteratively yields:
\begin{flalign}
&\mE[\ltwo{w_{t}-w^*(\theta_{t})}^2] \nonumber\\
&\leq \left( 1- \frac{\lambda_2\beta}{8}  \right)^t\ltwo{w_0 - w^*(\theta_0)}^2 +  \frac{100\lambda^2_1\alpha^2}{\lambda^2_2 \beta}  \sum_{i=0}^{t-1}\left(1- \frac{\lambda_2\beta}{8} \right)^i \mE[\ltwo{\nabla J(\theta_i)}^2] \nonumber\\
&\quad + 32(4R^2_\theta\rho^2_{\max} + r^2_{\max})\left(  \frac{32\alpha^2}{\lambda^2_2 \beta} + \frac{2\beta}{\lambda_2} + 2\beta^2 \right)\frac{1+(\kappa-1)\rho}{(1-\rho)M} \sum_{i=0}^{t-1}\left(1- \frac{\lambda_2\beta}{8} \right)^i  \nonumber\\
&\leq \left( 1- \frac{\lambda_2\beta}{8}  \right)^t\ltwo{w_0 - w^*(\theta_0)}^2 +  \frac{100\lambda^2_1\alpha^2}{\lambda^2_2 \beta} \sum_{i=0}^{t-1}\left(1- \frac{\lambda_2\beta}{8} \right)^i \mE[\ltwo{\nabla J(\theta_i)}^2] \nonumber\\
&\quad + \frac{256}{\lambda_2\beta}(4R^2_\theta\rho^2_{\max} + r^2_{\max})\left(  \frac{32\alpha^2}{\lambda^2_2 \beta} + \frac{2\beta}{\lambda_2} + 2\beta^2 \right)\frac{1+(\kappa-1)\rho}{(1-\rho)M}. \label{eq: 28}
\end{flalign}
Summing \cref{eq: 28} from $t=0$ to $T-1$ yields
\begin{flalign}
&\sum_{t=0}^{T-1}\mE[\ltwo{w_{t}-w^*(\theta_{t})}^2]\nonumber\\
&\leq \ltwo{w_0 - w^*(\theta_0)}^2 \sum_{t=0}^{T-1}\left( 1- \frac{\lambda_2\beta}{8}  \right)^t +  \frac{100\lambda^2_1\alpha^2}{\lambda^2_2 \beta} \sum_{t=0}^{T-1}\sum_{i=0}^{t-1}\left(1- \frac{\lambda_2\beta}{8} \right)^i \mE[\ltwo{\nabla J(\theta_i)}^2] \nonumber\\
&\quad + \frac{256T}{\lambda_2\beta}(4R^2_\theta\rho^2_{\max} + r^2_{\max})\left(  \frac{32\alpha^2}{\lambda^2_2 \beta} + \frac{2\beta}{\lambda_2} + 2\beta^2 \right)\frac{1+(\kappa-1)\rho}{(1-\rho)M}\nonumber\\
&\leq \frac{8}{\lambda_2\beta} \ltwo{w_0 - w^*(\theta_0)}^2 +  \frac{800\lambda^2_1 \alpha^2}{\lambda^3_2\beta^2  } \sum_{t=0}^{T-1} \mE[\ltwo{\nabla J(\theta_t)}^2] \nonumber\\
&\quad + \frac{256T}{\lambda_2\beta}(4R^2_\theta\rho^2_{\max} + r^2_{\max})\left(  \frac{32\alpha^2}{\lambda^2_2 \beta} + \frac{2\beta}{\lambda_2} + 2\beta^2 \right)\frac{1+(\kappa-1)\rho}{(1-\rho)M}.\label{eq: 29}
\end{flalign}
Substituting \cref{eq: 29} into \cref{eq: 27} yields
\begin{flalign}
&\Big(\frac{\alpha}{4} - \frac{L_J\alpha^2}{2} \Big)\sum_{t=0}^{T-1}\mE[\ltwo{\nabla J(\theta_t)}^2] \nonumber\\
&\leq J(\theta_0)-\mE[J(\theta_T)] + ( 2\alpha + L_J\alpha^2 ) \frac{8\rho^2_{\max}}{\lambda_2\beta}\ltwo{w_0 - w^*(\theta_0)}^2 + 2( \alpha + L_J\alpha^2 )T\frac{C_2[1+(\kappa-1)\rho]}{M(1-\rho)} \nonumber\\
&\quad + ( 2\alpha + L_J\alpha^2 ) \rho^2_{\max} \frac{800\lambda^2_1 \alpha^2}{\lambda^3_2\beta^2  } \sum_{t=0}^{T-1} \mE[\ltwo{\nabla J(\theta_t)}^2] \nonumber\\
&\quad + 2( \alpha + L_J\alpha^2 )\frac{64\lambda^2_1(\rho_{\max}+1)^4[1+(\kappa-1)\rho]}{(1-\rho)M}\sum_{t=0}^{T-1}\mE[\ltwo{\nabla J(\theta_t)}^2] \nonumber\\
&\quad + ( 2\alpha + L_J\alpha^2 ) \rho^2_{\max}\frac{256T}{\lambda_2\beta}(4R^2_\theta\rho^2_{\max} + r^2_{\max})\left(  \frac{32\alpha^2}{\lambda^2_2 \beta} + \frac{2\beta}{\lambda_2} + 2\beta^2 \right)\frac{1+(\kappa-1)\rho}{(1-\rho)M} \nonumber\\
&\overset{(i)}{\leq} J(\theta_0)-\mE[J(\theta_T)] + \frac{24\alpha\rho^2_{\max}}{\lambda_2\beta}\ltwo{w_0 - w^*(\theta_0)}^2 + 4\alpha T\frac{C_1[1+(\kappa-1)\rho]}{M(1-\rho)} \nonumber\\
&\quad + \frac{2656 \rho^2_{\max}\lambda^2_1 \alpha^3}{\lambda^3_2\beta^2  } \sum_{t=0}^{T-1} \mE[\ltwo{\nabla J(\theta_t)}^2], \label{eq: 30}
\end{flalign}
where in $(i)$ we let $\alpha\leq \frac{1}{L_J}$ and $M\geq \frac{\beta^2\lambda^3_2(\rho_{\max}+1)^4[1+(\kappa-1)\rho]}{\rho^2_{\max}\alpha^2(1-\rho)}$, and define $C_1=C_2 + \frac{192 \rho^2_{\max}}{\lambda_2\beta}(4R^2_\theta\rho^2_{\max} + r^2_{\max})\left(  \frac{32\alpha^2}{\lambda^2_2 \beta} + \frac{2\beta}{\lambda_2} + 2\beta^2 \right)$. Rearranging \cref{eq: 30} yields
\begin{flalign*}
&\left(\frac{\alpha}{4} - \frac{L_J\alpha^2}{2} - \frac{2656 \rho^2_{\max}\lambda^2_1 \alpha^3}{\lambda^3_2\beta^2  } \right)\sum_{t=0}^{T-1}\mE[\ltwo{\nabla J(\theta_t)}^2] \nonumber\\
&\leq J(\theta_0)-\mE[J(\theta_T)] + \frac{24\alpha\rho^2_{\max}}{\lambda_2\beta}\ltwo{w_0 - w^*(\theta_0)}^2 + 4\alpha T\frac{C_1[1+(\kappa-1)\rho]}{M(1-\rho)}.
\end{flalign*}
Letting $\alpha \leq \min\{ \frac{1}{8L_J}, \frac{L_J\lambda^3_2\beta^2}{5312\rho^2_{\max}\lambda^2_1}\}$, we obtain
\begin{flalign*}
& \frac{\alpha}{8}\sum_{t=0}^{T-1}\mE[\ltwo{\nabla J(\theta_t)}^2] \nonumber\\
&\leq J(\theta_0)-\mE[J(\theta_T)] + \frac{24\alpha\rho^2_{\max}}{\lambda_2\beta}\ltwo{w_0 - w^*(\theta_0)}^2 + 4\alpha T\frac{C_1[1+(\kappa-1)\rho]}{M(1-\rho)}.
\end{flalign*}
Dividing both sides of the above inequality by $\frac{\alpha T}{8}$ yields
\begin{flalign*}
\frac{1}{T}\sum_{t=0}^{T-1}\mE[\ltwo{\nabla J(\theta_t)}^2]\leq \frac{8(J(\theta_0)-\mE[J(\theta_T)])}{\alpha T} + \frac{192\rho^2_{\max}}{\lambda_2\beta}\frac{\ltwo{w_0 - w^*(\theta_0)}^2}{T} + \frac{32C_1[1+(\kappa-1)\rho]}{M(1-\rho)}.
\end{flalign*}

\end{document}